\documentclass[]{arxiv_mm}

\usepackage[T1]{fontenc}
\usepackage[utf8]{inputenc}
\usepackage{lipsum}
\usepackage{amsfonts}
\usepackage{graphicx}
\usepackage{epstopdf}
\usepackage{algorithmic}
\usepackage{booktabs}
\usepackage[normalem]{ulem}
\usepackage{cancel}
\usepackage{xcolor}

\usepackage{subcaption}

\usepackage{tikz}
\usetikzlibrary{positioning, arrows.meta, decorations.pathreplacing, calligraphy}
\usepackage{pgfplots}
\pgfplotsset{compat=1.17}

\usepackage{amsmath, amssymb} 
\usepackage{bm}              
\usepackage{algorithm}       
\usepackage{pifont}          

\usepackage{enumitem}
\setlist[enumerate]{leftmargin=.5in}
\setlist[itemize]{leftmargin=.5in}

\usepackage{hyperref}

\usepackage{cleveref}

\theoremstyle{definition}
\newtheorem*{theoreticalassumptions}{Theoretical Assumptions}

\crefformat{theoreticalassumptions}{#2theoretical assumptions#3}
\Crefformat{theoreticalassumptions}{#2Theoretical assumptions#3}

\newcommand{\tikzmark}[1]{\tikz[overlay,remember picture] \node (#1) {};}

\newcommand{\PP}{\mathbb{P}}
\newcommand{\EE}{\mathbb{E}}

\ifpdf
  \DeclareGraphicsExtensions{.eps,.pdf,.png,.jpg}
\else
  \DeclareGraphicsExtensions{.eps}
\fi

\newsiamremark{remark}{Remark}
\newsiamremark{hypothesis}{Hypothesis}
\crefname{hypothesis}{Hypothesis}{Hypotheses}
\newsiamthm{claim}{Claim}
\newsiamremark{fact}{Fact}
\crefname{fact}{Fact}{Facts}

\definecolor{darkblue}{rgb}{0,0,0.55}

\hypersetup{
  colorlinks=true,
  linkcolor=darkblue,   
  citecolor=darkblue,   
  urlcolor=darkblue     
}

\title{Iterative Atom Refinement: A Monotonicity Principle for Dictionary Learning}

\author{Alexander Christie\thanks{Department of Mathematics, Stanford University, Stanford, CA 94305, USA (\email{achristie@stanford.edu}).}
\and Miguel Moscoso\thanks{Department of Mathematics, Universidad Carlos III de Madrid, Legan\'es, Madrid 28911, Spain (\email{moscoso@math.uc3m.es}).}
\and Alexei Novikov\thanks{Department of Mathematics, Pennsylvania State University, University Park, PA 16802, USA (\email{novikov@psu.edu}).}
\and George Papanicolaou\thanks{Department of Mathematics, Stanford University, Stanford, CA 94305, USA (\email{papanicolaou@stanford.edu}).}
\and Chrysoula Tsogka\thanks{Department of Applied Mathematics, University of California, Merced, CA 95343, USA (\email{ctsogka@ucmerced.edu}).}}

\usepackage{amsopn}

\begin{document}

\maketitle

\begin{abstract}
Dictionary learning seeks to recover an unknown dictionary $A$ from observations
${\bf y}_i = A{\bf x}_i$ with sparse coefficient vectors ${\bf x}_i$. We introduce the
\emph{Iterative Atom Refinement} (IAR) algorithm, a simple procedure for
recovering individual dictionary atoms. Starting
from a random direction, IAR repeatedly selects the observations most strongly
correlated with the current iterate and updates the direction by averaging the
selected data. Our main contribution is a rigorous convergence theory of IAR. Using high-dimensional probabilistic estimates and a novel
monotonicity principle for atom-selection probabilities, we show that a small
initial advantage of one atom is amplified until that atom is isolated. Under
our model assumptions, IAR identifies a generating atom after only three
refinement steps. Numerical experiments support the theory and show that the
resulting dynamics accurately capture the behavior observed in dictionary
refinement.
\end{abstract}



\section{Introduction}
We introduce Iterative Atom Refinement (IAR), a new computationally inexpensive algorithm for dictionary learning that provably recovers dictionary atoms from random initialization.
 Our main contribution is theoretical. We identify a new \emph{monotonicity principle} for atom-selection frequencies. These one-dimensional quantities preserve their ordering under the refinement dynamics, while a small initial advantage of the leading atom is amplified until that atom is isolated. To the best of our knowledge, this mechanism has not previously appeared in dictionary learning.

Monotonicity is among the most powerful tools for uncovering structure in nonlinear problems. Once a suitable order or monotone quantity is found, comparison can replace explicit solution and force remarkably rigid behavior. This principle underlies, among other things, maximum and comparison principles, the moving-plane method, and convergence theory for monotone dynamical systems~\cite{ProtterWeinberger,GidasNiNirenberg,Hirsch}. Our result shows that an analogous structure is present in dictionary learning and leads, under our model assumptions, to recovery from random initialization. This contrasts with the available convergence guarantees for standard alternating dictionary-learning methods, which require initialization sufficiently close to the generating dictionary~\cite{Agarwal,Liang2025}, although convergence from imperfect or random initialization has been established for other iterative statistical procedures~\cite{Balakrishnan2017,WuZhou2021,Chandrasekher2024}.

 This mechanism may also help explain a phenomenon we observed in earlier neural-network experiments~\cite{Weak}. There, recovery from random initialization was achieved by combining multiple independent runs with clustering, despite the nonconvexity of the problem. IAR provides a tractable model in which such an amplification mechanism can be analyzed rigorously.

The setting is as follows. Let ${\bf A}=[{\bf a}_1,\dots,{\bf a}_K]\in\mathbb{R}^{N\times K}$ be an unknown dictionary and suppose the observed data satisfy
$$
{\bf Y} = {\bf A}\,{\bf X},
$$
where ${\bf Y}=[{\bf y}_1,\dots,{\bf y}_M]\in\mathbb{R}^{N\times M}$ is the observed data matrix and ${\bf X}=[{\bf x}_1,\dots,{\bf x}_M]\in\mathbb{R}^{K\times M}$ contains the unknown sparse coefficients. 
The sparsity $s:=\|{\bf x}_i\|_0$ is small relative to $K$. The goal of dictionary learning is to recover ${\bf A}$ from ${\bf Y}$.

Our algorithm, \emph{Iterative Atom Refinement} (IAR), recovers individual atoms of an unknown dictionary. Starting from a unit vector ${\bf d}^{(0)} \in \mathbb{R}^N$, IAR updates its iterate ${\bf d}^{(r)}$ at step $r$ by selecting the subset $S^{(r)} \subset \{1, \dots, M\}$ of size $L = M/K$ that exhibits the strongest correlations $\langle {\bf d}^{(r)}, {\bf y}_i \rangle$. The next iterate is generated by averaging these top-correlated samples and projecting back onto the unit sphere, so
\[
{\bf d}^{(r+1)} = \frac{\sum_{i \in S^{(r)}} {\bf y}_i}{\left\| \sum_{i \in S^{(r)}} {\bf y}_i \right\|_2}.
\]
Thus, each iteration of IAR reduces to a simple sequence of correlation thresholding, sample averaging, and normalization.

For the analysis, we introduce the \emph{atom-selection frequencies}
\[
p_k^{(r+1)}
=
\frac{1}{L}
\#\left\{
i\in S^{(r)}:\ k\in s_i
\right\}
\]
for each dictionary atom ${\bf a}_k$, where \# denotes cardinality and $s_i=\operatorname{supp}({\bf x}_i)$. Thus, $p_k^{(r+1)}$ is the proportion of selected samples containing the $k$-th dictionary atom. These quantities are not observable by the algorithm, since the supports $s_i$ are unknown; they are introduced only for the analysis. Under our theoretical model, they satisfy a key monotonicity principle: if
\[
p_k^{(1)} \geq p_{k+1}^{(1)},
\qquad k=1,\ldots,K-1,
\]
then this ordering is preserved under the refinement dynamics, so
\[
p_k^{(r)} \geq p_{k+1}^{(r)},
\qquad k=1,\ldots,K-1,\quad \text{for all} \quad r\geq 1.
\]
More importantly, a strict initial advantage in atom-selection frequency is amplified from one refinement step to the next until the corresponding atom dominates the selected samples and is isolated.

For orthogonal dictionaries under the full combinatorial data model, our main result shows that, with high probability over the random initialization, IAR identifies after three refinement steps the dictionary atom having 
the largest initial correlation. Consequently, $R$ independent random initializations recover the entire dictionary with probability at least $1-e^{-t}$ provided
\[
R\ge K(\log K+t).
\]
Numerical experiments show the same rapid convergence for partial data and nonorthogonal dictionaries.

Our results place IAR among provable correlation-based methods for dictionary learning. Arora, Ge, and Moitra~\cite{Arora} construct a graph by thresholding pairwise sample correlations and recover an overlapping clustering that identifies individual dictionary atoms, yielding polynomial-time recovery guarantees for incoherent overcomplete dictionaries at sparsity levels near the square-root scale. Bhaskara and Tai~\cite{BhaskaraTai} isolate threshold correlation as a primitive for approximate dictionary recovery. Novikov and White~\cite{NovikovWhite} use weighted covariance matrices to estimate support subspaces and then intersect these subspaces to recover individual atoms, reaching an almost-linear sparsity regime.

Optimization-based methods form another major line of provable dictionary learning. For standard practical alternating methods~\cite{MOD,KSVD}, the rigorous convergence guarantees are local. Other global-recovery methods use different mechanisms. For complete invertible dictionaries, Spielman, Wang, and Wright~\cite{Spielman} reduce dictionary learning to finding sparse vectors in the row space of the data and recover the dictionary through a sequence of linear programs. Sum-of-squares and tensor methods provide another route to global recovery~\cite{Barak2015}.

IAR belongs to the correlation-based line of work, but uses correlations dynamically rather than to reconstruct a global graph, clustering, or support structure. 
The relation to optimization-based dictionary learning and the empirical origins of IAR~\cite{Strong} are discussed in \Cref{sec:greedyMOD}.

The paper is organized as follows. \Cref{sec3} introduces IAR and states the full dictionary recovery guarantee. \Cref{sec:convergence} develops the monotonicity principle and the convergence theory, including single-atom convergence and the proof of full dictionary recovery. \Cref{sec:greedyMOD} explains the imaging motivation for IAR, reviews the initialization barrier for optimization-based dictionary-learning methods, and derives IAR from the algorithmic ideas that motivated it. \Cref{sec:numerics} presents numerical experiments testing the theory and comparing IAR with standard dictionary-learning methods. The proofs are given in \Cref{iar_thm,tech_lemma}, followed by conclusions in \Cref{sec:conclusions}. Finally, \Cref{app:successive-pk} gives the exact asymptotic scaling of the IAR selection probabilities.

\section{Iterative Atom Refinement}\label{sec3}

In this section, we introduce the IAR algorithm. We first describe the
one-atom refinement procedure and then explain how independent random
initializations are combined to recover the full dictionary. We state the
resulting recovery guarantee here. Its analysis is given in
\Cref{sec:convergence}.

\subsection{Single atom recovery with IAR}

\paragraph{Algorithm (One-atom IAR)}
Let $\{ {\bf y}_i \}_{i=1}^{M}$ be a collection of data samples in $\mathbb{R}^N$ such that ${\bf y}_i = {\bf A}\, {\bf x}_i$, and ${\bf x}_i\in\mathbb{R}^K$
is $s$-sparse.
The goal is to iteratively refine a unit-norm ${\bf d}^{(r)} \in\mathbb{R}^N$
so that it aligns with a true dictionary column $\mathbf{a}_k$ in ${\bf A}$.

\begin{enumerate}
  \item[\textbf{(1)}] \textbf{Initialization.}  
  Draw an initial column vector ${\bf d}^{(0)}\in\mathbb{R}^N$ uniformly at random
  from the unit sphere.

  \medskip

  \item[\textbf{(2)}] \textbf{Selection threshold.}  
  For the current iterate ${\bf d}^{(r)}$, $r=0,1,2, \dots$, determine a threshold $\alpha_r > 0$
  such that the selected set
  \begin{equation}\label{sy}
  S^{(r)} := \bigl\{ i : \langle {\bf d}^{(r)}, {\bf y}_i\rangle
                              \ge \alpha_r \bigr\}
  \end{equation}
  has cardinality (For simplicity, in the theoretical analysis we assume that $M/K$ is an integer.)
  \[
  |S^{(r)}| = L := \frac{M}{K}.
  \]
  That is, $\alpha_r$ selects the top $L=M/K$ samples ${\bf y}_i$ that most strongly correlate with the current direction ${\bf d}^{(r)}$. In practice, 
  we do not compute or estimate $\alpha_r$. Instead, we just order $\langle {\bf d}^{(r)}, {\bf y}_i\rangle$ in decreasing form and select the first $L=M/K$ samples.

  Another key element of the algorithm is that $\alpha_r$ is not fixed but adaptive. At each step it selects the $L=M/K$
  most correlated samples with the iterate ${\bf d}^{(r)}$.

  \medskip

  \item[\textbf{(3)}] \textbf{Coefficient update.}  
  Construct a coefficient vector ${\boldsymbol \chi}^{(r)} \in \mathbb{R}^M$ with components
  \begin{equation}\label{threshold}
  \chi^{(r)}_i =
  \begin{cases}
  1, & \text{if } i \in S^{(r)}, \\[4pt]
  0, & \text{otherwise.}
  \end{cases}
  \end{equation}
  This defines a sparse vector whose nonzero entries correspond
  to the selected set $S^{(r)}$. This update for a row
  of ${\bf X}$ can be equivalently written in compact form using the
  hard thresholding operator as
  \[
  {\boldsymbol \chi}^{(r)} = 
  \mathcal{T}_{\alpha_{r}}\!\big({\bf Y}^T{\bf d}^{(r)}\big).
  \]
  Here $\mathcal{T}_{\alpha_{r}}({\bf Y}^T{\bf d}^{(r)})$ produces a binary 
  vector selecting all samples ${\bf y}_i$ whose correlation
  $\langle {\bf y}_i, {\bf d}^{(r)} \rangle$ exceeds~$\alpha_{r}$.
  The threshold~$\alpha_{r}$ controls the sparsity level of the resulting vector.

  \medskip

  \item[\textbf{(4)}] \textbf{Dictionary update.}  
  Update the atom by aggregating the selected data samples via
  \[
  {\bf d}^{(r+1)} = {\bf Y} {\boldsymbol \chi}^{(r)},
  \]
  and subsequently normalize the resulting vector to unit length, so
  \[
  {\bf d}^{(r+1)} \gets \frac{{\bf d}^{(r+1)}}{\|{\bf d}^{(r+1)}\|_2}.
  \]

  \medskip

  \item[\textbf{(5)}] \textbf{Iteration and stopping rule.}  
  Repeat steps (2)--(4) until convergence, i.e.,
  \[
  \| {\bf d}^{(r+1)} - {\bf d}^{(r)} \|_2 < \varepsilon,
  \]
  for some tolerance $\varepsilon > 0$, or until the direction stabilizes, so
  \[
  | \langle {\bf d}^{(r+1)}, {\bf d}^{(r)} \rangle |
  \ge 1-\delta,
  \quad \text{with $\delta$ small.}
  \]
\end{enumerate}

A single run is designed to isolate one dictionary atom. To recover the
entire dictionary, we repeat the procedure from independent random
initializations.

\subsection{Full dictionary recovery with IAR}

The single-atom analysis in \Cref{sec:convergence} shows that, under the assumptions
stated there, each independent initialization of IAR discovers a dictionary
atom in exactly three steps. By symmetry, the discovered atom is uniformly
distributed among the $K$ dictionary atoms. Thus, full dictionary recovery
reduces to the classical coupon collector problem.

This leads to an expected number of independent trials on the order of
$K\log K$ for full collection via the classical coupon collector's problem.

Once a run isolates a dictionary atom, its corresponding row of ${\bf X}$ can also be recovered. Under the theoretical model, the samples containing the isolated atom are exactly the $\binom{K-1}{s-1}$ samples having the largest correlations with the terminal iterate; see the proof of \Cref{fd_r}. After all $K$ rows of ${\bf X}$ have been recovered, the dictionary is obtained from
\[
{\bf A}
=
{\bf Y}{\bf X}^T
\bigl({\bf X}{\bf X}^T\bigr)^{-1}.
\]
Thus, the coupon-collector argument determines the number of independent runs needed to isolate all $K$ atoms.
When the total dictionary size $K$ is unknown, this probability framework provides a practical stopping criterion for full dictionary recovery with IAR. As $R$ grows, the probability of hitting an {\em unseen} atom diminishes exponentially. Thus, if the number of newly discovered dictionary atoms drops below a predetermined threshold over a sequence of recent trials of size $B$, the algorithm terminates. Specifically, we count the number of \emph{new, previously unseen} atoms discovered within a current batch of $B$ trials and stop the algorithm once this count falls below the predetermined threshold.

For the following theorem, we assume the model conditions introduced in
\Cref{sec:convergence}. These are precisely the conditions used to analyze
a single IAR run. Full dictionary recovery adds only one ingredient: the
random initializations of the different runs are independent.

\begin{theorem}[Full dictionary recovery with random initialization]\label{fd_r}
Perform \(R\) independent
runs of IAR with independent uniform random initializations.

For any \(t>0\), if
\[
R\ge K(\log K+t),
\]
then the probability that the full dictionary \({\bf A}\) is recovered is at least
\[
1-e^{-t}.
\]

Equivalently, taking \(t=\log(1/\delta)\), with
\[
R \geq K\bigl(\log K+\log(1/\delta)\bigr),
\]
the full dictionary is recovered with probability at least \(1-\delta\).
\end{theorem}

The proof of \Cref{fd_r} is given at the end of
\Cref{sec:convergence}.

\section{Monotonicity and convergence}\label{sec:convergence}

We now turn to the main mechanism behind this result, a monotonicity principle for the atom-selection frequencies.
Under the full combinatorial data model, an ordering of the correlations with the dictionary atoms is preserved by the refinement step. 
We then show that a random initialization creates, with high probability, a sufficient separation among the leading atoms, and use these two facts to prove convergence to a single atom and, ultimately, full dictionary recovery.

\subsection{Monotonicity principle}

For our theoretical analysis, we consider the \emph{full combinatorial data} matrix $\mathbf{X} = [\mathbf{x}_1, \dots, \mathbf{x}_M] \in \{0,1\}^{K \times M}$, where $M = \binom{K}{s}$. The columns of $\mathbf{X}$ enumerate every binary $s$-sparse vector in $\{0,1\}^K$ exactly once. Explicitly, for each $s$-element subset $\{i_1, \dots, i_s\} \subset \{1, \dots, K\}$, $\mathbf{X}$ contains a unique column $\mathbf{x} = (x_1, \dots, x_K)^T$ defined by
\[
x_k = \begin{cases} 
1, & k \in \{i_1, \dots, i_s\}, \\ 
0, & \text{otherwise.} 
\end{cases}
\]
Throughout this section we assume that ${\bf A}^T{\bf A}={\bf I}_K$.

The aim of the IAR procedure is to turn an arbitrary unit vector
${\bf d}^{(0)}\in\mathbb R^N$ into a vector aligned with a single dictionary column.
Write
\[
\tilde{\bf d}^{(0)} := {\bf A}^T{\bf d}^{(0)} \in\mathbb R^K
\]
and, without loss of generality, assume its components are strictly decreasing, so
\begin{equation}\label{dees}
\tilde{d}^{(0)}_1 > \tilde{d}^{(0)}_2 > \tilde{d}^{(0)}_3 > \cdots > \tilde{d}^{(0)}_K.
\end{equation}
If the  components are strictly decreasing, then we expect that ${\bf d}^{(r)} \to {\bf a}_1$, a column of $A$ such that  $A^T {\bf a}_1 = {\bf e}_1= (1,0,0,\dots,0)^T$ (Assume that ${\bf A}^T{\bf A}={\bf I}_K$). 

The key observation is that the ordering in~\eqref{dees} is preserved at each iteration. In other words, if ${\bf \tilde d}^{(0)}$ satisfies~\eqref{dees}, then at each iteration $r$ we have
\begin{equation}\label{dees_r}
\tilde{d}^{(r)}_1 \ge \tilde{d}^{(r)}_2 \ge \tilde{d}^{(r)}_3 \ge \cdots \ge \tilde{d}^{(r)}_K.
\end{equation}

To see this, let \# denote the cardinalities of sets, and let ${\bf x}_i$ be the sparse coefficient vector associated with the data sample $\mathbf{y}_i$. We denote its support by $s_i = \mathrm{supp}\,{\bf x}_i$. The main counting\footnote{The quantities $p_k^{(r)}$ are unobservable during the selection step. They are revealed during the evaluation of the next iterate of ${\bf d}^{(r)}$.} quantity is the {\it atom-selection frequency }
\begin{equation}\label{miguel1}
p_{k}^{(r+1)} = \frac{
  \# \left\{ i :\ k \in s_i,\ \langle {\bf {d}}^{(r)}, {\bf y}_i \rangle > \alpha_r\right\}
}{
  L
}.
\end{equation}
Here, $\langle {\bf {d}}^{(r)}, {\bf y}_i \rangle > \alpha_r$ is the threshold criterion that defines the selected set~\eqref{sy}, and $k \in s_i$ means $k$ is in the support of ${\bf x}_i$ and ${\bf y}_i= {\bf A} {\bf x}_i$.
Thus, $p_k^{(r+1)}$ is the fraction of the $L$ selected samples in which the $k$-th atom appears in the sparse representation at iteration $r+1$. 

By construction, each iterate ${\bf d}^{(r)}$, $r \geq 1$ lies in the span of the dictionary atoms. Hence, there exists a coefficient vector ${\bf p}^{(r)}=\bigl[p_1^{(r)},\dots,p_K^{(r)}\bigr]^T$ such that
\[
{\bf d}^{(r)}=\gamma_r\, {\bf A}\,{\bf p}^{(r)},
\]
where $\gamma_r>0$ is chosen so that $\|{\bf d}^{(r)}\|_2=1$. More specifically, observe that
\[
{\bf d}^{(r+1)} = \gamma' \sum_{i = 1}^{M} \chi_i^{(r)} {\bf y}_i
=\gamma' {\bf A} \sum_{i = 1}^{M} \chi_i^{(r)} {\bf x}_i
= \gamma \,  {\bf A} \,  {\bf p}^{(r+1)}.
\]
where $\chi_i^{(r)}$ is the indicator function \eqref{threshold} for sample selection at iteration $r$.

Therefore, ${\bf \tilde d}^{(r)} = {\bf A}^T {\bf d}^{(r)} = \text{Const}\cdot{\bf p}^{(r)}$ and monotonicity of ${\bf p}^{(r)}$
is equivalent to monotonicity of ${\bf \tilde d}^{(r)}$. For the full data model, ${\bf p}^{(r+1)}$ can be interpreted as a vector of conditional probabilities
\begin{equation}\label{full_data}
p_{k}^{(r+1)} = \text{Prob}\left\{ k \in s \;\middle|\; \langle {\bf d}^{(r)}, {\bf y} \rangle > \alpha_r\right\}.
\end{equation}

The non-strict  monotonicity of the probabilities $p_k^{(r+1)} \ge p_{k+1}^{(r+1)}$ can be established through the following counting argument.
The key point is that, under full data, the truncation step always
selects at least as many samples containing index $k$ as samples containing index $k+1$,
provided
\[
\tilde d^{(r)}_k \ge \tilde d^{(r)}_{k+1}.
\]
Compare cardinalities of the two sets
\[
S_k^{(r)}(\alpha_r)
   := \bigl\{ i:\ k \in s_i,\ \langle {\bf d}^{(r)}, {\bf y}_i\rangle > \alpha_r \bigr\} \quad \text{and} \quad S_{k+1}^{(r)}(\alpha_r)
   := \bigl\{ i:\ k+1 \in s_i,\ \langle {\bf d}^{(r)}, {\bf y}_i\rangle > \alpha_r \bigr\}.
\]
If there were no truncation (i.e.\ $\alpha_r=-\infty$), then by symmetry of the full
data set we would have
\[
\# S_k^{(r)}(-\infty) = \# S_{k+1}^{(r)}(-\infty).
\]
The only samples that can break this parity under a finite threshold
$\alpha_r$ are those whose supports distinguish between $k$ and $k+1$.
There are exactly two such disjoint types:
\[
\text{(i)}\quad 
k\in s_i,\ k+1\notin s_i,
\qquad\text{and}\qquad
\text{(ii)}\quad 
k\notin s_i,\ k+1\in s_i.
\]
Now take any sample ${\bf x}_i$ of type~(ii), i.e.,
\[
x_{k,i}=0,\qquad x_{k+1,i}=1.
\]
Construct from it a new vector $\widehat{\bf x}_i$ obtained by
``moving the 1'' from index $k+1$ to index $k$, so
\[
\widehat x_{k,i}=1,\qquad \widehat x_{k+1,i}=0,
\qquad\text{and}\qquad
\widehat x_{\ell,i}=x_{\ell,i}\ \text{for }\ell\neq k,k+1.
\]
Because we are using the full data set, every admissible support
configuration appears exactly once; hence $\widehat{\bf x}_i$ is also one of the
columns of ${\bf X}$. Since $\tilde{d}^{(r)}_{k} \ge \tilde{d}^{(r)}_{k+1}$, we obtain
\[
\langle {\bf d}^{(r)}, \widehat{\bf y}_i \rangle - \langle {\bf d}^{(r)}, {\bf y}_i \rangle 
= \langle {\bf \tilde d}^{(r)}, \widehat{\bf x}_i \rangle - \langle {\bf \tilde d}^{(r)}, {\bf x}_i \rangle 
= \tilde{d}^{(r)}_{k} - \tilde{d}^{(r)}_{k+1} \geq 0,
\]
where $\widehat{\bf y}_i:=A\widehat{\bf x}_i$.
If $\langle {\bf d}^{(r)}, {\bf y}_i \rangle > \alpha_r$, then $\langle {\bf d}^{(r)}, \widehat{\bf y}_i \rangle > \alpha_r$.
This guarantees
\[
\#\, S_k^{(r)}(\alpha_r) \geq \#\, S_{k+1}^{(r)}(\alpha_r),
\]
establishing the non-strict  monotonicity of the selection probabilities under full data.

At this point one would like to conclude that $p_{1}^{(r)} \to 1$ as $r \to \infty$. Unfortunately, the non-strict monotonicity of  $p_{k}^{(r)}$
alone is not enough to arrive at this conclusion. 
More precisely, to force $p_{1}^{(r)}$ to dominate in the limit we need \emph{both} (i) preservation of the ordering \(p_k^{(r)}\ge p_{k+1}^{(r)}\) at every
iteration, \emph{simultaneously for all indices $k$}, and (ii) quantitative lower bounds on the gaps $p_{k}^{(1)}-p_{k+1}^{(1)}$, in the first iteration,  for the \emph{first finitely many} indices $k$.
Thus, we need stronger estimates on the size of the gap between successive proportions \(p_{k}^{(1)}\) and \(p_{k+1}^{(1)}\) for the full data. 
Recall that (see~\eqref{full_data})
for the full data  we can express these proportions as probabilities $p_{1}$, $\dots$, $p_{K}$, respectively, where we omitted the superscript $(1)$. 

\subsection{Separation after the first refinement}

We now define these probabilities precisely and quantify their separation.

Assume that the initial vector ${\bf d}^{(0)}$ is chosen uniformly at random
from the unit sphere in $\mathbb{R}^{N}$.  Since the dictionary has
orthonormal columns, the distribution of
\[
\widetilde{\bf d}^{(0)}
=
{\bf A}^{T}{\bf d}^{(0)}
\]
is rotationally invariant in $\mathbb{R}^{K}$.  Moreover, all the estimates
below are invariant under multiplication of $\widetilde{\bf d}^{(0)}$ by a
positive scalar.  We may therefore replace its normalized, rotationally
invariant distribution by an unnormalized standard Gaussian vector.  Thus,
after omitting the superscript $(0)$, we may assume that
\begin{equation}\label{normal_t}
    \widetilde d_j
    \overset{\mathrm{i.i.d.}}{\sim}
    \mathcal{N}(0,1),
    \qquad
    j=1,\dots,K.
\end{equation}

Let ${\bf x}$ be a uniformly chosen column of the full combinatorial data
matrix ${\bf X}\in\{0,1\}^{K\times M}$, chosen independently of
$\widetilde{\bf d}$.  Every column of ${\bf X}$ has exactly $s$ nonzero
entries.  Hence, 
\begin{equation}\label{main_y}
    z
    =
    \langle \widetilde{\bf d},{\bf x}\rangle
    \sim
    \mathcal{N}(0,s).
\end{equation}
Since
\[
    \langle {\bf d},{\bf y}\rangle
    =
    \langle {\bf d},{\bf A}{\bf x}\rangle
    =
    \langle {\bf A}^{T}{\bf d},{\bf x}\rangle
    =
    \langle \widetilde{\bf d},{\bf x}\rangle.
\]

Let $\Phi$ denote the standard Gaussian CDF, set
\[
q_K:=\Phi^{-1}\!\left(1-\frac1K\right).
\]
and define the deterministic reference threshold
\[
\bar\alpha_0:=\sqrt{s}\,q_K.
\]
Since $z/\sqrt{s}$ is standard Gaussian,
\[
\mathbb P\{z\ge\bar\alpha_0\}=\frac1K.
\]
The actual threshold $\alpha_0$ used by IAR depends on the realization of
$\widetilde{\bf d}$ and is chosen so that exactly a $1/K$ fraction of the
full data is selected. In the proof of \Cref{lemma4} we will show that
\[
\alpha_0=\bar\alpha_0+o_{\mathbb P}(q_K^{-1}).
\]

The usual Gaussian-tail asymptotics give
\[
\bar\alpha_0
=
\sqrt{s}\left(
\sqrt{2\log K}
-
\frac{\log\log K+\log(4\pi)}
     {2\sqrt{2\log K}}
+
o\bigl((\log K)^{-1/2}\bigr)
\right),
\]
and, in particular,
\[
\bar\alpha_0
\sim
\sqrt{2s\log K}
\qquad
\text{as }K\to\infty.
\]

Finally, by permutation symmetry of the model, we may relabel the dictionary
atoms and assume without loss of generality that
\[
    \widetilde d_1
    \geq
    \widetilde d_2
    \geq
    \cdots
    \geq
    \widetilde d_K.
\]

Fix a realization of $\widetilde{\bf d}$, and let ${\bf x}$ be a uniformly
chosen column of the full combinatorial data matrix. For each $k$, let
$z_k$ denote the conditional random variable
\[
z_k
=
\langle\widetilde{\bf d},{\bf x}\rangle
\quad\text{given}\quad
k\in\operatorname{supp}({\bf x}),
\]
where the only randomness is in the choice of ${\bf x}$. Set
\[
p_k
:=
\mathbb P_{\bf x}\!\left(
z_k>\alpha_0
\,\middle|\,
\widetilde{\bf d}
\right).
\]
Thus $p_k$ is itself a random variable through its dependence on the random
initialization $\widetilde{\bf d}$.

Recall that for the full data model $p_k^{(1)}$ are probabilities,
see~\eqref{full_data}. Since
\[
\mathbb P_{\bf x}\bigl(k\in\operatorname{supp}({\bf x})\bigr)=\frac{s}{K}
\]
and the selected set has probability $1/K$, we have
\[
p_k^{(1)}=s\,p_k.
\]
Hence ratios between the $p_k^{(1)}$ are the same as ratios between the
$p_k$. We use this equivalence in the following key Lemma.

\begin{lemma}[Separation after the first refinement]\label{lemma4}
Fix \(s\ge2\). For every \(\eta>0\), there exist
\(c\in(0,1)\), \(\varepsilon>0\), and an integer \(\kappa\)
such that
\[
c+(s-1)\varepsilon<1
\]
and, for all sufficiently large \(K\), with probability at least
\(1-\eta\) over the random initialization,
\[
p_2^{(1)}\le c\,p_1^{(1)}
\qquad\text{and}\qquad
p_\kappa^{(1)}\le \varepsilon p_1^{(1)}.
\]
Consequently, by monotonicity,
\[
p_k^{(1)}\le \varepsilon p_1^{(1)}
\qquad\text{for all }k\ge\kappa.
\]
\end{lemma}

We defer the proof of this Lemma to \Cref{tech_lemma}. 

\subsection{Single-atom convergence}

We first state the assumptions used in the convergence theorem.

\begin{theoreticalassumptions}\label{assump:theoretical_assumptions}
Let $K\le N$ and
$
{\bf A}=[{\bf a}_1,\dots,{\bf a}_K]\in\mathbb{R}^{N\times K},
\,
{\bf A}^T{\bf A}={\bf I}_K.
$
The data matrix is
\[
{\bf Y}={\bf A}{\bf X},
\]
where ${\bf X}\in\{0,1\}^{K\times M}$ is the full combinatorial data matrix
whose columns are all binary vectors with exactly $s$ nonzero entries. Thus,
$
M=\binom{K}{s}.
$

The IAR iteration is
\[
{\bf d}^{(r+1)}
=
\frac{{\bf Y}{\boldsymbol\chi}^{(r)}}
{\|{\bf Y}{\boldsymbol\chi}^{(r)}\|_2},
\qquad
{\boldsymbol\chi}^{(r)}
=
\mathcal{T}_{\alpha_r}\bigl({\bf Y}^T{\bf d}^{(r)}\bigr),
\]
where $\alpha_r$ is chosen so that the associated selection set $S^{(r)}$
has cardinality $L=M/K$.

Finally, we assume a random initialization $\mathbf{d}^{(0)}$ drawn
uniformly from the unit sphere.
\end{theoreticalassumptions}

\begin{remark}
The orthogonality assumption is not essential from a geometric point of view.
Indeed, if ${\bf A}$ has full column rank, the Löwdin orthogonalization
\[
\widehat{\bf A}
=
{\bf A}({\bf A}^T{\bf A})^{-1/2}
\]
has orthonormal columns. Equivalently, after an invertible linear change of
coordinates on $\operatorname{span}{\bf A}$, the data can be written as
\[
\widehat{\bf y}_i
=
\widehat{\bf A}{\bf x}_i
\]
with the same sparse coefficient vectors ${\bf x}_i$. Thus the orthogonality
assumption can be viewed as a normalization of the dictionary. 

A related transformation can be estimated directly from the data by whitening.
Indeed, if the coefficient covariance is proportional to the identity, then
the inverse square root of ${\bf Y}{\bf Y}^T$ restricted to
$\operatorname{span}{\bf Y}$
maps the dictionary to an orthogonal one, up to normalization of its columns.
More generally, when the coefficient covariance is close to a scalar multiple
of the identity, the resulting dictionary is correspondingly close to
orthogonal.
\end{remark}

\begin{theorem}[Convergence of Iterative Atom Refinement]
\label{thm:single_atom}
Let ${\bf d}^{(0)}$ denote the initial unit vector of the IAR iteration.
Under the theoretical assumptions stated above, with high probability over
the random choice of ${\bf d}^{(0)}$,
\[
p_1^{(r)}\longrightarrow 1
\qquad\text{as } r\to\infty,
\]
where $p_k^{(r)}$ is defined in~\eqref{full_data}. Thus, IAR asymptotically
isolates the dictionary atom with the largest initial correlation with
${\bf d}^{(0)}$.
\end{theorem}

\begin{remark}\label{rem:finite-step}
Under the theoretical assumptions above, we in fact prove the stronger finite-step result
\[
p_1^{(r)}=1,
\qquad r\ge 3.
\]
This exact finite-step convergence relies on the highly structured setting considered here, in particular on the orthogonality of the dictionary, the binary coefficients, and the full combinatorial data set. Nevertheless, our numerical experiments indicate that the same qualitative behavior persists in considerably more general and practically relevant settings: the fraction of selected samples containing the dominant atom rapidly approaches one and remains close to one after only a few refinement steps. 
\end{remark}

The proof of \Cref{thm:single_atom} is given in \Cref{iar_thm}.

\subsection{Full dictionary recovery}

We now use the stronger finite-step conclusion in \Cref{rem:finite-step}
to prove the full dictionary recovery theorem, \Cref{fd_r}.

\begin{proof}[Proof of \Cref{fd_r}]
By permutation symmetry, the atom singled out by a uniform random initialization is uniformly distributed over \(\{1,\dots,K\}\), independently from run to run.
Consider one run and, after relabeling, suppose that this atom is \({\bf a}_1\). For any \(r\ge3\), we have \(p_1^{(r)}=1\). Write
\[
{\bf d}^{(r)}=\frac{{\bf A}{\bf p}^{(r)}}{\|{\bf p}^{(r)}\|_2},
\qquad
{\bf p}^{(r)}=(1,p_2^{(r)},\dots,p_K^{(r)})^T.
\]
Since every sample used to construct \({\bf p}^{(r)}\) contains index \(1\), for every \(k\ge2\),
\[
p_k^{(r)}
\le
\frac{\binom{K-2}{s-2}}{M/K}
=
\frac{s(s-1)}{K-1}.
\]
If \(1\in\operatorname{supp}({\bf x}_i)\), then
\[
\langle{\bf d}^{(r)},{\bf y}_i\rangle
=
\frac{\langle{\bf p}^{(r)},{\bf x}_i\rangle}{\|{\bf p}^{(r)}\|_2}
\ge
\frac1{\|{\bf p}^{(r)}\|_2},
\]
whereas, if \(1\notin\operatorname{supp}({\bf x}_i)\), then
\[
\langle{\bf d}^{(r)},{\bf y}_i\rangle
\le
\frac{s\max_{k\ge2}p_k^{(r)}}{\|{\bf p}^{(r)}\|_2}
\le
\frac{s^2(s-1)}{(K-1)\|{\bf p}^{(r)}\|_2}.
\]
For sufficiently large \(K\), these two ranges are disjoint. Since exactly \(\binom{K-1}{s-1}\) columns of \({\bf X}\) contain index \(1\), the \(\binom{K-1}{s-1}\) samples with the largest correlations with \({\bf d}^{(r)}\) are precisely those for which \(x_{1i}=1\). Thus the entire row of \({\bf X}\) corresponding to the isolated atom is recovered exactly. The same argument applies to every atom isolated by an IAR run.
It remains to isolate all \(K\) atoms. The probability that a fixed atom is missed in \(R\) independent runs is \((1-1/K)^R\); hence, by the union bound, the probability that at least one atom is missed is at most
\[
K\left(1-\frac1K\right)^R
\le
K e^{-R/K}.
\]
For \(R=K(\log K+t)\), this is at most \(e^{-t}\). Therefore, with probability at least \(1-e^{-t}\), all rows of \({\bf X}\) are recovered. Since the full combinatorial matrix \({\bf X}\) has full row rank,
\[
{\bf A}
=
{\bf Y}{\bf X}^T
\bigl({\bf X}{\bf X}^T\bigr)^{-1},
\]
and the dictionary is recovered exactly.
\end{proof}

\section{Overcoming the Initialization Barrier}\label{sec:greedyMOD}

We first describe the imaging problems that led us to a greedy modification of the Method of Optimal Directions (MOD)~\cite{MOD}, then discuss why existing optimization theory does 
not explain the behavior of this modification, and finally show how IAR emerges as a tractable one-atom model of this modification.

\subsection{Motivation}
Our interest in dictionary learning comes from inverse problems in imaging, where the sensing operator may itself be unknown and must be inferred from the data~\cite{CIO-DL,Weak,Strong}. This is closely connected with blind deconvolution, where one seeks to recover both the unknown forward response and the underlying signal~\cite{ayers1988iterative,KundurHatzinakos1996}.

In~\cite{CIO-DL}, MOD was effective because a reliable initial approximation to the unknown dictionary could be constructed from prior information. In more difficult imaging settings~\cite{Weak}, this initial approximation was no longer sufficiently accurate, and obtaining a suitable initialization required an additional computationally expensive procedure~\cite{NovikovWhite}.

In~\cite{Weak}, we also showed that a neural-network-based method could recover the dictionary from random initialization. What was particularly striking was how it overcame the nonconvexity of the problem: multiple independent random initializations produced different candidate atoms, which were then combined through clustering to recover the dictionary. This behavior was quite different from standard MOD and from the theoretical mechanisms available to us at the time. The neural-network approach was, however, substantially more computationally expensive, and its recovery mechanism was difficult to analyze.

This led us in~\cite{Strong} to \emph{greedy MOD}, which reproduced much of the same behavior in a considerably simpler setting. At each stage, the algorithm favors the coefficients with the strongest current correlations and immediately updates the dictionary accordingly. We call this choice \emph{greedy} because it acts on the current preference rather than first solving the sparse-coding problem to convergence. The monotonicity principle proved above for IAR now provides a plausible explanation for this phenomenon: a small random preference for one atom is preserved and amplified until that atom is isolated, suggesting why repeated random initializations can produce the different atoms needed for full dictionary recovery.

To make this distinction precise, standard MOD~\cite{MOD} alternates between sparse coding and dictionary updates to solve the non-convex optimization problem
\begin{equation}\label{eq:mod_obj}
\min_{\mathbf A,\mathbf X}|\mathbf Y-\mathbf A\mathbf X|_F^2
\quad\text{subject to}\quad |\mathbf x_i|_0\le s.
\end{equation}
For comparison with greedy MOD, we consider standard MOD with ISTA~\cite{ISTA} for sparse coding. Other sparse-coding solvers can also be used, but ISTA makes the distinction particularly clear: standard MOD runs sparse coding to convergence before updating the dictionary, whereas greedy MOD performs only one sparse-coding iteration.

\begin{minipage}{0.85\textwidth}
\begin{algorithm}[H]
\caption{Standard MOD with ISTA sparse coding}
\label{alg:mod}
\begin{algorithmic}[1]
\STATE \textbf{Input}: Data matrix $\mathbf{Y} \in \mathbb{R}^{N \times M}$,
threshold parameter $\tau$, outer stopping criterion $\epsilon > 0$,
inner stopping criterion $\epsilon_{\mathrm{in}} > 0$, and target dictionary size $K$.
\STATE \textbf{Output}: Learned dictionary
$\mathbf{D}_{\mathrm{end}} \in \mathbb{R}^{N \times K}$.
\STATE Draw $\mathbf{X}^{(0)} \in \mathbb{R}^{K \times M}$ with i.i.d.\ standard Gaussian entries.
\STATE Compute
$\mathbf{D}^{(0)} = \mathbf{Y}\big(\mathbf{X}^{(0)}\big)^\dagger$,
normalize its columns to unit length, and set $n=0$.
\WHILE{$\|\mathbf{Y}-\mathbf{D}^{(n)}\mathbf{X}^{(n)}\|_F>\epsilon$}
    \STATE Set $L_n=\|\mathbf{D}^{(n)}\|_2^2$,
    $\mathbf{Z}^{(0)}=0$, and $j=0$.
    \REPEAT
        \STATE
        $\widetilde{\mathbf{Z}}^{(j+1)}
        =
        \mathbf{Z}^{(j)}
        +
        \dfrac{1}{L_n}
        \big(\mathbf{D}^{(n)}\big)^\top
        \big(\mathbf{Y}-\mathbf{D}^{(n)}\mathbf{Z}^{(j)}\big)$
        \STATE
        $\mathbf{Z}^{(j+1)}
        =
        \operatorname{sign}\big(\widetilde{\mathbf{Z}}^{(j+1)}\big)
        \odot
        \max\left(
        \big|\widetilde{\mathbf{Z}}^{(j+1)}\big|
        -\dfrac{\tau}{L_n},0
        \right)$
        \STATE $j=j+1$
    \UNTIL{$\|\mathbf{Z}^{(j)}-\mathbf{Z}^{(j-1)}\|_F\le\epsilon_{\mathrm{in}}$}
    \STATE Set $\mathbf{X}^{(n+1)}=\mathbf{Z}^{(j)}$.
    \STATE
    \begin{equation}\label{up1}
    \widetilde{\mathbf{D}}^{(n+1)}
    =
    \mathbf{Y}\big(\mathbf{X}^{(n+1)}\big)^\dagger.
    \end{equation}
    \STATE Set $\mathbf{D}^{(n+1)}$ as
    $\widetilde{\mathbf{D}}^{(n+1)}$ with columns normalized to unit length.
    \STATE $n=n+1$
\ENDWHILE
\STATE Set $\mathbf{D}_{\mathrm{end}}=\mathbf{D}^{(n)}$.
\STATE \textbf{return} $\mathbf{D}_{\mathrm{end}}$.
\end{algorithmic}
\end{algorithm}
\end{minipage}

In our experiments, standard MOD with random initialization usually fails to recover the generating dictionary. This is consistent with the available convergence guarantees, which are local and assume an initialization sufficiently close to the generating dictionary~\cite{Agarwal,Liang2025}; see \Cref{alg:mod}.

Greedy MOD~\cite{Strong} uses the same ISTA sparse-coding step and the same
dictionary update~\eqref{up1}, but performs only one ISTA iteration before
each dictionary update. Thus, the dictionary is updated immediately according
to the current preference of the sparse-coding step, rather than after sparse
coding has converged; see \Cref{alg:greedy_mod}.

\begin{minipage}{0.85\textwidth}
\begin{algorithm}[H]
\caption{Greedy MOD: Single ISTA Iteration, Random Restarts, and Clustering}
\label{alg:greedy_mod}
\begin{algorithmic}[1]
\STATE \textbf{Input}: Data matrix $\mathbf{Y} \in \mathbb{R}^{N \times M}$,
threshold parameter $\tau$, stopping criterion $\epsilon > 0$,
target dictionary size $K$, and number of restarts $\aleph$.
\STATE \textbf{Output}: Final learned dictionary
$\mathbf{D}_{\mathrm{end}} \in \mathbb{R}^{N \times K}$.
\STATE \textbf{Initialize}: Atom collection pool
$\boldsymbol{\mathcal{C}}\leftarrow\emptyset$.
\FOR{$r=1\text{ to } \aleph$}
    \STATE \tikzmark{top-mark}Draw
    $\mathbf{X}^{(0)} \in \mathbb{R}^{K \times M}$ with i.i.d.\ standard Gaussian entries.
    \STATE Compute
    $\mathbf{D}^{(0)}=\mathbf{Y}\big(\mathbf{X}^{(0)}\big)^\dagger$,
    normalize its columns to unit length, and set $n=0$.
    \WHILE{$\|\mathbf{Y}-\mathbf{D}^{(n)}\mathbf{X}^{(n)}\|_F>\epsilon$}
        \STATE Set $L_n=\|\mathbf{D}^{(n)}\|_2^2$.
        \STATE
        $\widetilde{\mathbf{X}}^{(n+1)}
        =
        \dfrac{1}{L_n}
        \big(\mathbf{D}^{(n)}\big)^\top\mathbf{Y}$
        \STATE
        $\mathbf{X}^{(n+1)}
        =
        \operatorname{sign}\big(\widetilde{\mathbf{X}}^{(n+1)}\big)
        \odot
        \max\left(
        \big|\widetilde{\mathbf{X}}^{(n+1)}\big|
        -\dfrac{\tau}{L_n},0
        \right)$
        \STATE
        $\widetilde{\mathbf{D}}^{(n+1)}
        =
        \mathbf{Y}\big(\mathbf{X}^{(n+1)}\big)^\dagger$
        \STATE Set $\mathbf{D}^{(n+1)}$ as
        $\widetilde{\mathbf{D}}^{(n+1)}$ with columns normalized to unit length.
        \STATE $n=n+1$
    \ENDWHILE\tikzmark{bot-mark}%
    \STATE Add all $K$ columns of the converged matrix
    $\mathbf{D}^{(n)}$ to the candidate pool $\boldsymbol{\mathcal{C}}$.
\ENDFOR
\STATE Cluster candidate atoms in $\boldsymbol{\mathcal{C}}$ based on mutual coherence.
\STATE Select the $K$ dominant cluster centroids to form the columns of
$\mathbf{D}_{\mathrm{end}}$.
\STATE \textbf{return} $\mathbf{D}_{\mathrm{end}}$.
\end{algorithmic}
\end{algorithm}

\begin{tikzpicture}[remember picture, overlay]
  \coordinate (bracketTop) at ([xshift=34em, yshift=1.5ex]top-mark);
  \coordinate (bracketBot) at (bracketTop |- bot-mark);
  
  \draw[decorate, decoration={brace, amplitude=9pt}, very thick] 
    (bracketTop) -- ([yshift=-0.5ex]bracketBot)
    node[midway, xshift=18pt, anchor=west, align=center, font=\small\bfseries]
    {MOD with \\ one ISTA iteration};
\end{tikzpicture}
\end{minipage}

Empirically, this simple modification changes the behavior dramatically: a single run rapidly recovers most dictionary atoms, while a small number of random restarts followed by clustering recovers the complete dictionary. Individual runs can nevertheless trap some atoms in spurious local minima (\Cref{fig:alignment_comparison} (b)), which motivates the use of multiple restarts and post-hoc clustering (\Cref{fig:alignment_comparison} (c)).

\subsection{Limits of existing optimization theory}
We next ask whether existing optimization theory explains why greedy MOD can recover dictionary atoms from random initialization. 
The available rigorous guarantees for practical alternating methods such as MOD~\cite{MOD} and K-SVD~\cite{KSVD} are local.
 Schnass~\cite{Schnass} established local identifiability conditions for the objective underlying K-SVD, Agarwal et al.~\cite{Agarwal} proved local convergence of a MOD-type alternating-minimization method for incoherent overcomplete dictionaries, and Liang et al.~\cite{Liang2025} proved linear convergence of a simple MOD-like alternating method for complete dictionary learning, again assuming a sufficiently accurate initialization.

Other optimization-based approaches do not explain the behavior of greedy MOD either. Zhai et al.~\cite{Zhai2020a} proposed an $\ell^4$-maximization problem over the orthogonal group together with an efficient iterative method, but their general convergence result still requires local initialization. Sun, Qu, and Wright~\cite{Wright} do obtain recovery from arbitrary initialization, but for a different nonconvex objective and with a substantially different algorithm based on Riemannian trust-region optimization, linear-programming rounding, and deflation.

Thus, existing optimization theory does not explain why such a simple MOD-type iteration can recover dictionary atoms from random initialization.

\subsection{IAR as a one-atom model of greedy MOD}

 IAR may be viewed as a simplified and analytically tractable model of greedy MOD~\cite{Strong}. It is obtained through two main simplifications.

The first simplification in IAR is that, rather than updating the entire dictionary at each iteration, it follows the evolution of a single dictionary atom. This reduction makes the dynamics substantially easier to analyze rigorously while preserving the essential refinement mechanism of greedy MOD. The reduction follows from the approximation
\[
{\bf X}{\bf X}^{T} \approx c\,{\bf I},
\qquad c>0.
\]
Then
\[
\big({\bf X}{\bf X}^{T}\big)^{-1}
\approx c^{-1}{\bf I},
\]
and the unnormalized MOD update~\eqref{up1} simplifies to
\[
\widetilde{\bf D}^{(n+1)}
\approx
\frac{1}{c}\,{\bf Y}{\bf X}^{(n)T}.
\]
Thus, the dictionary atoms decouple approximately and update independently as
\[
{\bf d}_k^{(n+1)}
\propto
{\bf Y}\,{\bf x}_k^{(n)T},
\]
where ${\bf x}_k^{(n)}$ is the $k$th row of ${\bf X}^{(n)}$. In this approximation, the global least-squares dictionary update is therefore replaced by independent averaging operations for the individual atoms.

A second simplification concerns the thresholding rule used in the sparse-coding step. Greedy MOD uses soft thresholding, whereas IAR uses hard thresholding. This removes the continuous shrinkage of the coefficients and makes the selection mechanism, and hence the analysis, considerably more transparent. Moreover, the adaptive threshold in IAR is chosen so that exactly a prescribed fraction of the samples is retained, selecting those most strongly correlated with the current iterate and directing the dynamics toward a single generating atom.

\section{Numerical experiments}\label{sec:numerics}

In this section, we empirically verify the key theoretical results of \Cref{thm:single_atom} and~\Cref{fd_r}, namely single-atom convergence,  selection frequency dynamics, and full dictionary recovery via random restarts. We then extend our evaluation to more general scenarios, such as data generated by non-orthogonal matrices, and benchmark IAR against both greedy MOD and standard MOD under identical experimental conditions.

\subsection{Theoretical Validation}

The numerical experiments in this subsection follow the assumptions of \Cref{thm:single_atom} and \Cref{fd_r}, using an orthogonal dictionary $\mathbf{A}$ alongside the full combinatorial sparse code matrix $\mathbf{X} \in \{0,1\}^{K \times M}$ with $s$-sparse binary columns.

\Cref{fig:iar_alignment} (left) tracks the evolution of the maximum alignment score, $\max_{1\le k\le K} |\langle \mathbf{d}^{(r)}, \mathbf{a}_k \rangle|$, across eight independent random initializations for $K = 300$ and $s = 3$. At $r = 0$, scores range from $0.15$ to $0.30$, reflecting the expected maximum inner product between a uniform random vector on $\mathcal{S}^{N-1}$ and $K$ orthonormal atoms. Within two iterations ($r=1, 2$), alignment surges to $0.90$--$0.97$, reaching a stable plateau of $0.9876$ by $r=3$. This small residual misalignment arises because the selected data samples contain interfering secondary atoms that do not cancel out completely. 

\begin{figure}[htbp]
\centering
\includegraphics[width=0.35\textwidth]{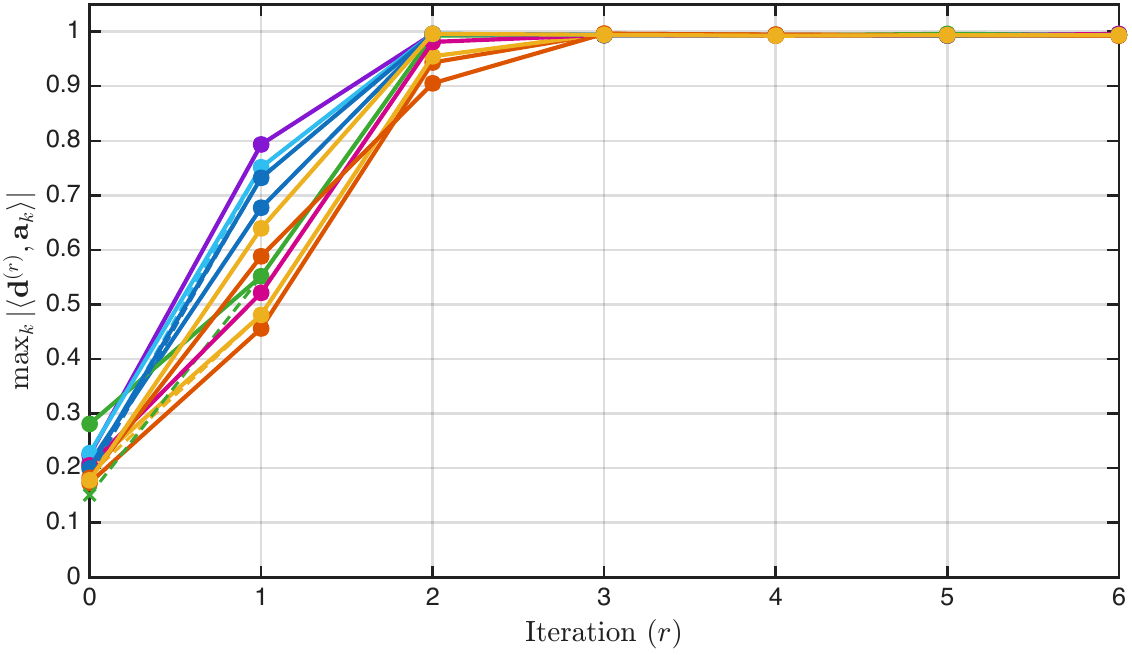} 
\includegraphics[width=0.32\textwidth]{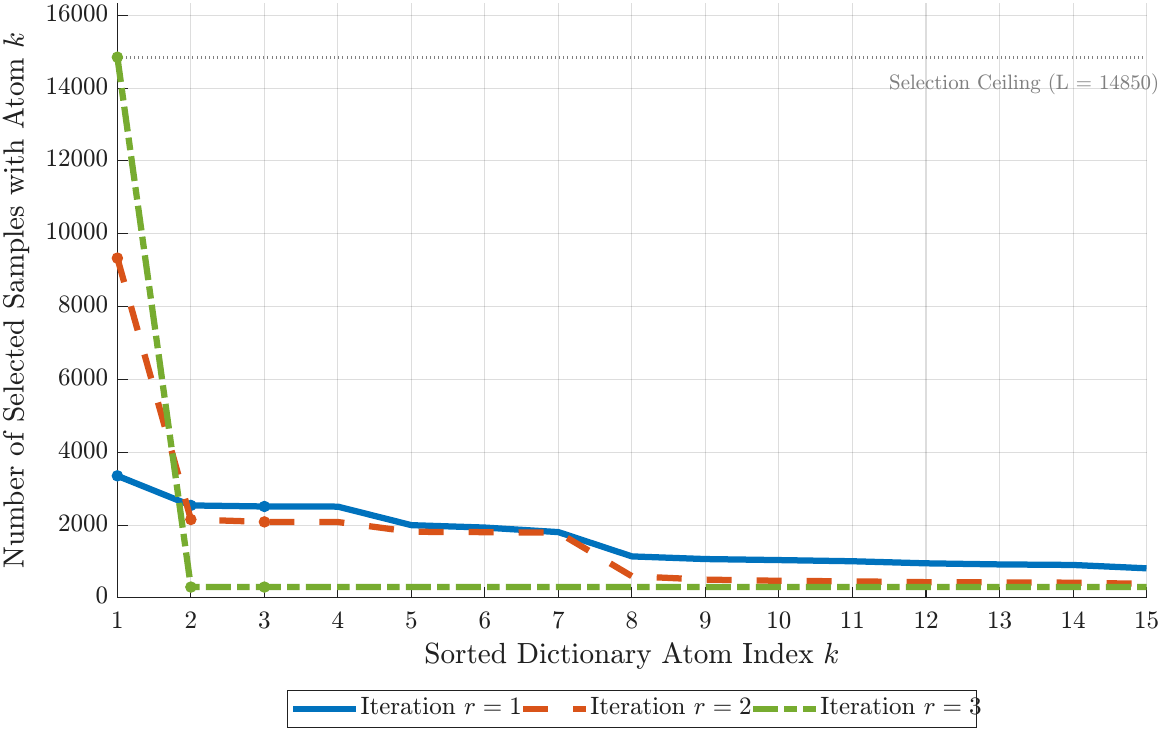}
\caption{Left: Alignment recovery 
over eight independent random initializations. Right: Ordered atom-selection profiles across the initial three iterations. 
}
    \label{fig:iar_alignment}
\end{figure}

\Cref{fig:iar_alignment} (right) tracks the sorted empirical atom-selection frequencies 
\begin{equation}
\widehat{p}_k^{(r)} := \frac{1}{L} \sum_{i \in \mathcal{S}^{(r)}} \mathbb{I}(k \in \operatorname{supp}(\mathbf{x}_i)),
\end{equation}
 across the first three iterations for the same simulation. 
The horizontal axis ranks dictionary atoms by descending selection frequency, while the vertical axis shows the corresponding sample counts $L \cdot \widehat{p}_k^{(r)}$ for samples clearing the threshold $\alpha_r$.
At iteration $r = 1$ (blue solid line), weak iterate alignment produces a nearly flat frequency curve across all atoms. By $r = 2$ (orange dashed line), sample aggregation triggers self-reinforcing alignment, opening a distinct gap between the target atom $k_{\star}=1$ and the background. By $r = 3$ (green dash-dotted line), the profile stabilizes. Note that, while the target atom completely dominates the primary position, background frequencies for $k \ge 2$ do not drop to zero. Instead, they flatten into a low-level persistent tail, causing the minor misalignment seen in  the left figure. 

\subsection{Empirical Extensions and Benchmarks}

To evaluate IAR beyond idealized conditions, we test its performance under limited sample sizes and non-orthogonal dictionaries, and benchmark against standard and greedy MOD. We begin by assessing recovery completeness under sample scarcity.

\Cref{fig:FigureC_CoverageCurve_K100} (left) tracks the empirical coverage fraction across independent restarts $R$ using only $1\%$ of the full dataset ${\bf Y}$ ($K = 500$, $s = 3$). Each trial yields a terminal vector ${\bf d}^{(r_\text{end})}$ that identifies an atom index via
$
\hat{k}=\arg\max_{1\le k\le K} |\langle {\bf d}^{(r_\text{end})},{\bf a}_k\rangle|\,,
$
generating the empirical cumulative coverage trajectory
\begin{equation}
R\longmapsto \frac{1}{K}\#\{\,\hat{k}_1,\dots,\hat{k}_R\,\}.
\end{equation}
The empirical trajectory (solid blue line) closely aligns with the coupon-collector baseline $1 - (1 - 1/K)^R$ (dashed red line), maintaining steady initial discovery before tapering off to achieve full ($100\%$) dictionary recovery at $R=3235$, very near the expected theoretical threshold $K \log K = 3107$.
\Cref{fig:FigureC_CoverageCurve_K100} (right) presents a more severe regime, with an extremely sparse data budget comprising only $0.05\%$ of the full dataset for  $K = 500$, $s = 4$.  Even under this drastic reduction in sample size, the algorithm maintains full $100\%$ dictionary coverage.

\begin{figure}[htbp]
    \centering
    \includegraphics[width=0.4\textwidth]{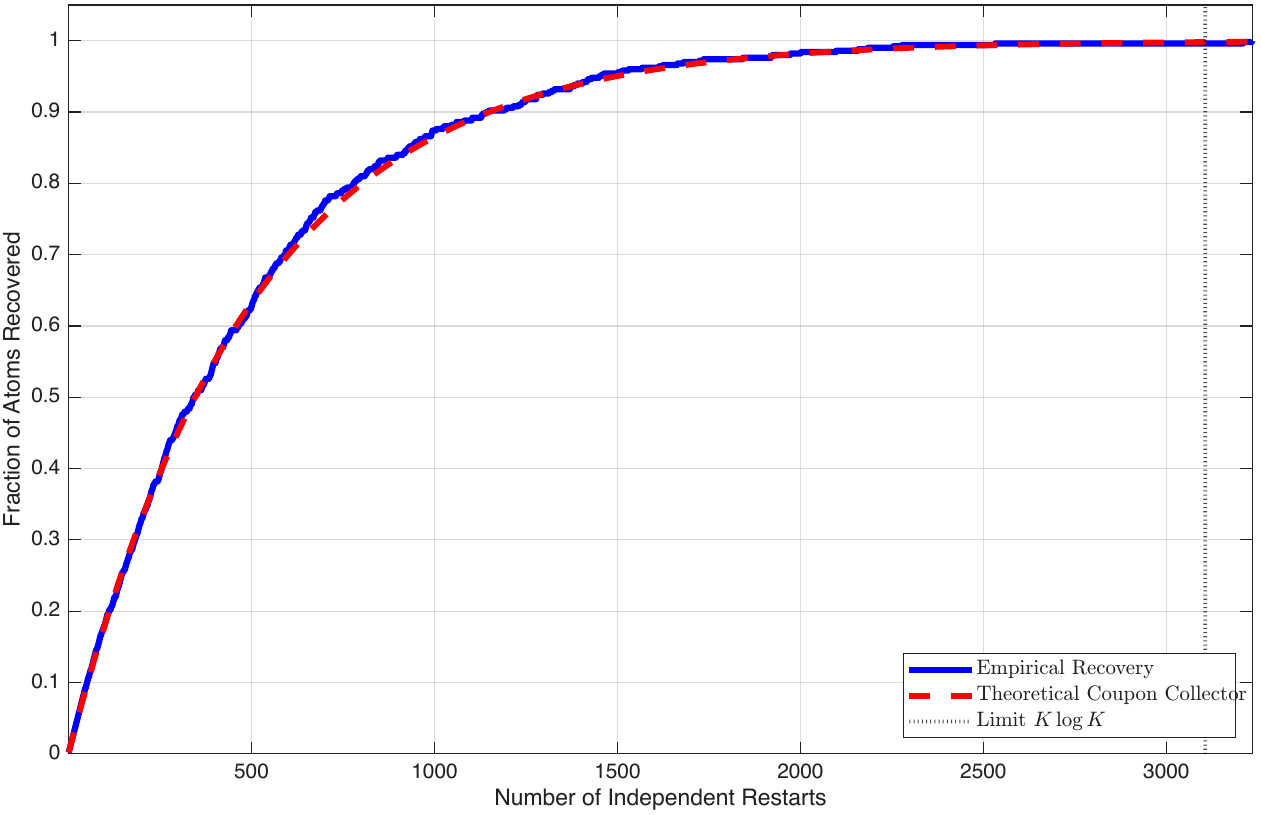} 
    \includegraphics[width=0.4\textwidth]{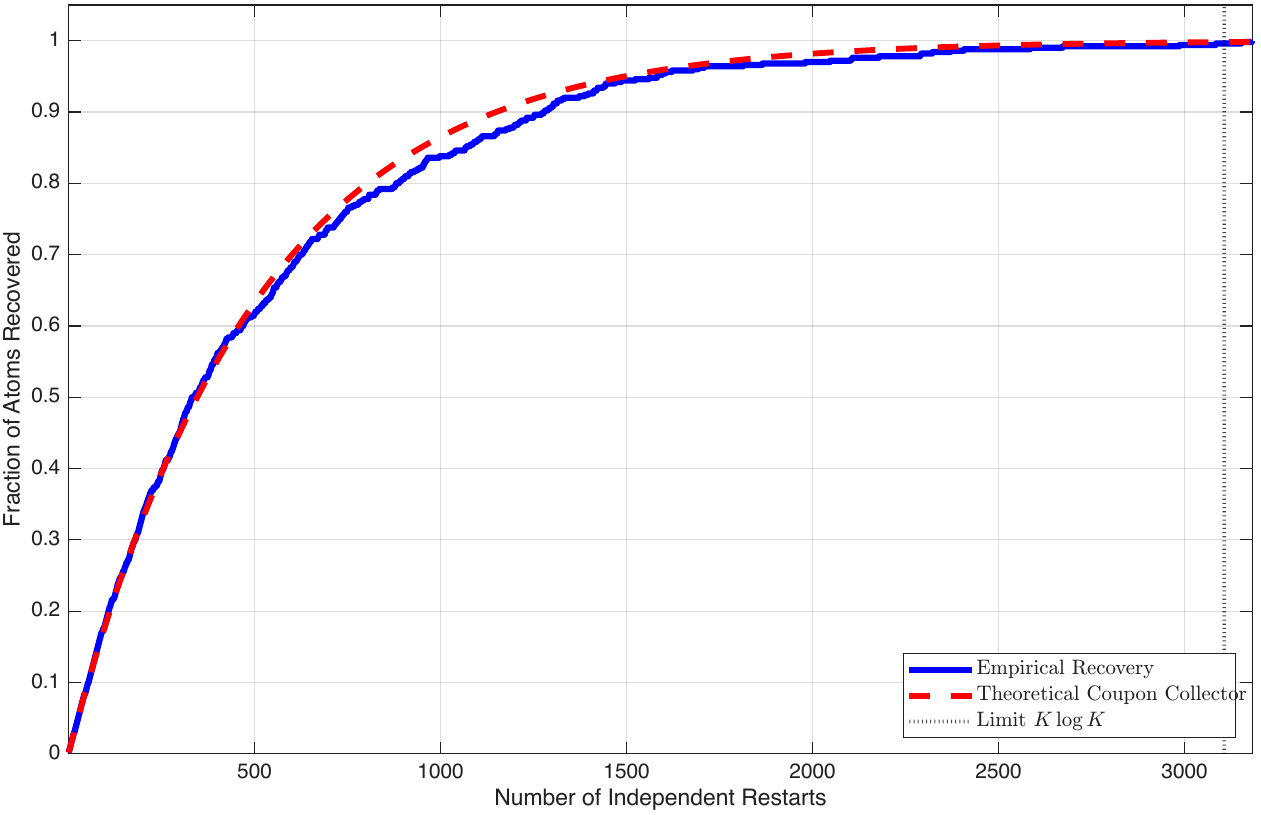}
    \caption{Left: Dictionary coverage curve versus independent restarts $R$  benchmarked against the theoretical Coupon Collector bound $1 - (1 - 1/K)^R$ and transition threshold $K \log K$; $K=500$ $(s=3)$. Right: Same as the left figure but with $0.05\%$ of the full data and $s=4$.
}
    \label{fig:FigureC_CoverageCurve_K100}
\end{figure}

\Cref{fig:tall_unperturbed} evaluates performance in the overdetermined setting  over $20$ independent runs using Gaussian dictionaries with aspect ratio $K/N \in [0.1, 1.0]$. Across all aspect ratios, the algorithm consistently converges in $5$ to $6$ iterations. Although theoretical model assumptions guarantee convergence at step 3, the termination condition $\Vert \mathbf{d}^{(r)} - \mathbf{d}^{(r-1)} \Vert < \epsilon$ requires at least $4$ iterations, with dictionary non-orthogonality adding $1$ to $2$ extra steps. \Cref{fig:tall_perturbed} shows comparable results under  random Gaussian coefficients $x_j=1 + \sigma z_j$, $\sigma = 0.5$, $z_j \sim \mathcal{N}(0,1)$. 

\begin{figure*}[htbp]
    \centering
    \begin{subfigure}[t]{0.31\textwidth}
        \centering
        \includegraphics[width=\textwidth]{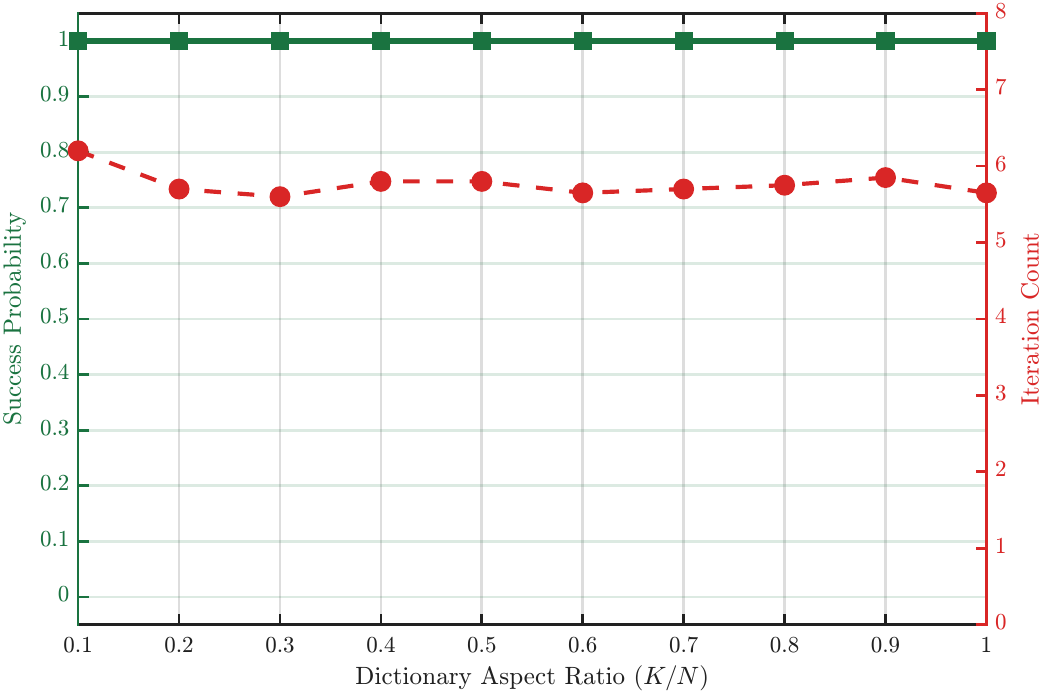}
        \caption{}
        \label{fig:tall_unperturbed}
    \end{subfigure}
    \hfill
    \begin{subfigure}[t]{0.31\textwidth}
        \centering
        \includegraphics[width=\textwidth]{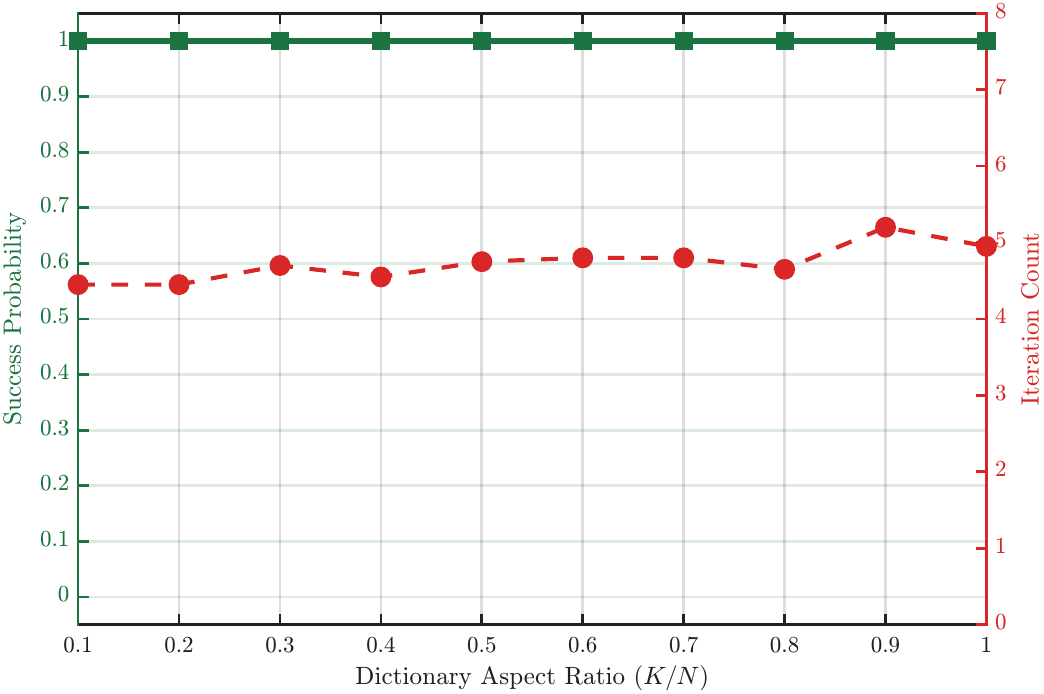}
        \caption{}
        \label{fig:tall_perturbed}
    \end{subfigure}
    \hfill
    \begin{subfigure}[t]{0.31\textwidth}
        \centering
        \includegraphics[width=\textwidth]{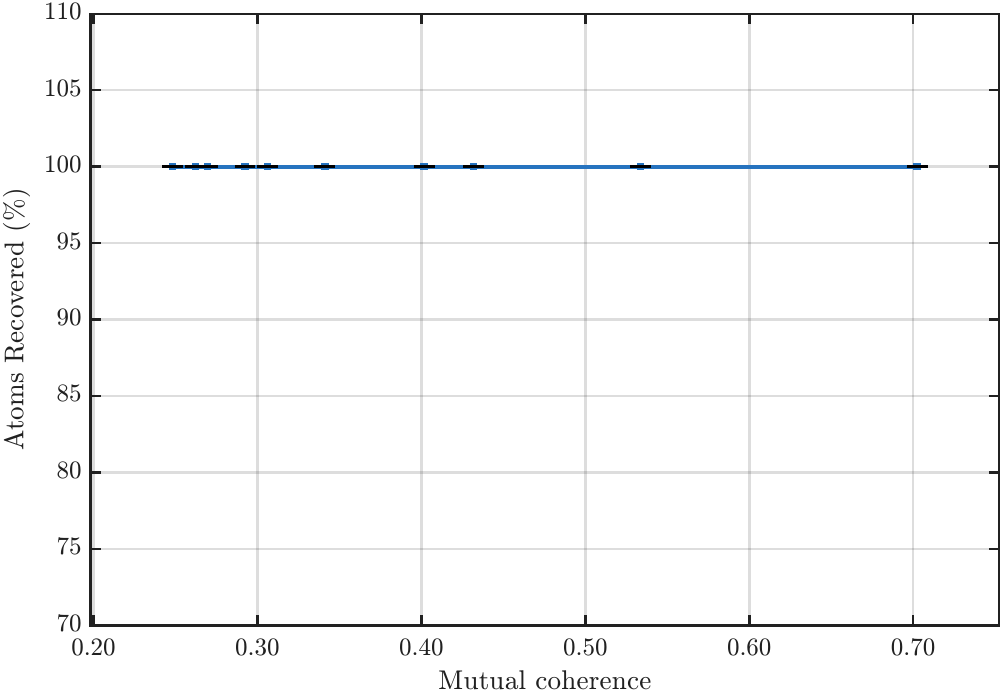} 
        \caption{}
        \label{fig:fat_dictionaries_stat}
    \end{subfigure}
    \caption{Performance metrics for $N \times K$ Gaussian dictionaries. Panels (a) and (b): Tall dictionaries ($N > K$, with $K=500$) evaluated over 20 trials across ratio $K/N \in [0.1, 1.0)$, showing single-atom empirical success probability (solid green, left axis) and average iterations to convergence (dashed red, right axis) under (a) unperturbed ($x_j = 1$) and (b) perturbed ($x_j = 1 + 0.5z_j$, $z_j \sim \mathcal{N}(0,1)$) coefficients. Panel (c): Fat dictionaries ($N < K$), plotting full-dictionary recovery percentage against mutual coherence $\mu(\mathbf{A})$ across 10 realizations.}
    \label{fig:overall_performance}
\end{figure*}

\Cref{fig:fat_dictionaries_stat} demonstrates full dictionary recovery in the underdetermined regime across $10$ independent trials. As the dimension $N$ of a random Gaussian dictionary ($K = 300$) sweeps from $300$ down to $30$, mutual coherence increases from $\mu(\mathbf{A}) = 0.25$ to $0.70$. Across all dimensions, IAR achieves $100\%$ dictionary recovery with a coupon-collector factor of $c = 2.0$. The narrow standard deviation whiskers ($\pm 1\,\text{SD}$) underscore the algorithm's robustness.


Finally, \Cref{fig:alignment_comparison} benchmarks IAR (bottom right) against standard MOD (top left), standard MOD with a single sparse-coding step (top right), and greedy MOD (bottom left). We plot the maximum correlation $\langle {\bf a}_k, {\bf d}_j \rangle$ between each true atom ${\bf a}_k$ and its closest recovered counterpart ${\bf d}_j$. Experiments use  data generated from an orthogonal dictionary ($K = 500$, $s = 5$) under severe sample scarcity ($M = 25000 \ll \binom{K}{s}$). All algorithms are initialized from the same Gaussian random matrix ${\bf D}_{0}$.
As shown in \Cref{fig:standard_mod}, standard MOD fails completely when sparse coding runs to convergence. However, limiting sparse coding to a single step dramatically improves performance, yielding near-perfect recovery across many dictionary atoms (\Cref{fig:greedy_mod_1ini}). Still, relying on a single initialization leaves this variant vulnerable to local minima, leaving $5\%$ to $15\%$ of atoms unrecovered depending on the initialization.
By incorporating five random initializations alongside clustering to identify consistently recovered atoms, greedy MOD eliminates local-minimum failures and achieves perfect $1.0$ correlation across all $500$ true atoms with complete stability (\Cref{fig:greedy_mod_5ini}).
IAR also achieves full recovery, maintaining a uniform correlation near $0.965$ (\Cref{fig:iar}). Because single-atom refinement bypasses matrix inversions, the algorithm runs over an order of magnitude faster. 


\begin{figure*}[htbp]
    \centering
    \begin{subfigure}[t]{0.23\textwidth}
        \centering
        \includegraphics[width=\textwidth]{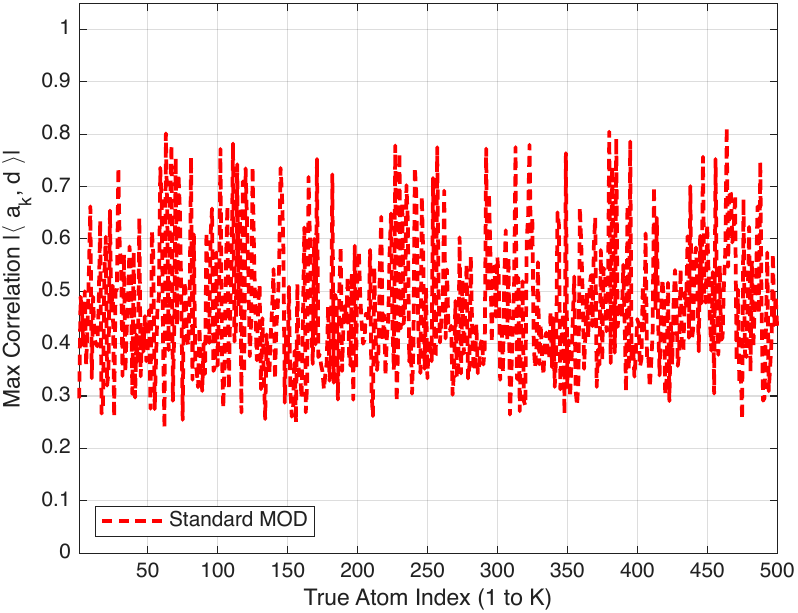}
        \caption{Standard MOD}
        \label{fig:standard_mod}
    \end{subfigure}
    \hfill
    \begin{subfigure}[t]{0.23\textwidth}
        \centering
        \includegraphics[width=\textwidth]{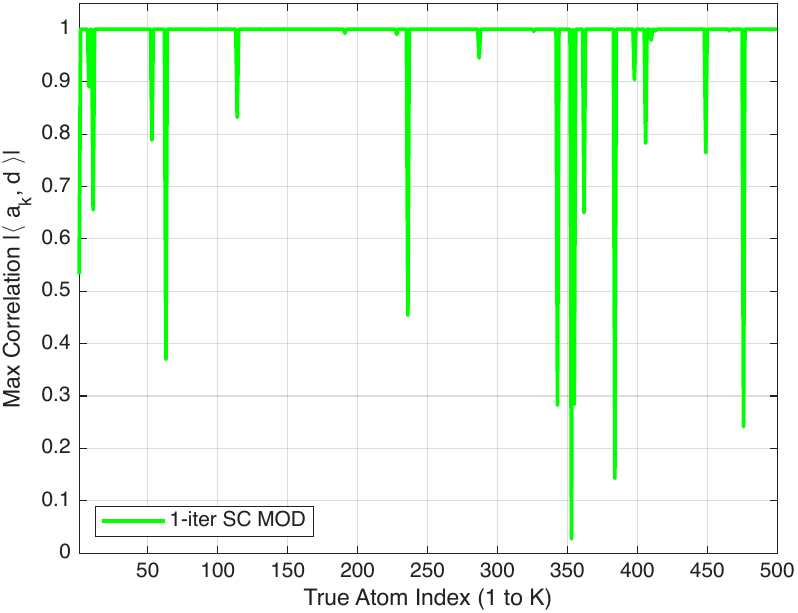}
        \caption{1-iter SC MOD}
        \label{fig:greedy_mod_1ini}
    \end{subfigure}
    \hfill
    \begin{subfigure}[t]{0.23\textwidth}
        \centering
        \includegraphics[width=\textwidth]{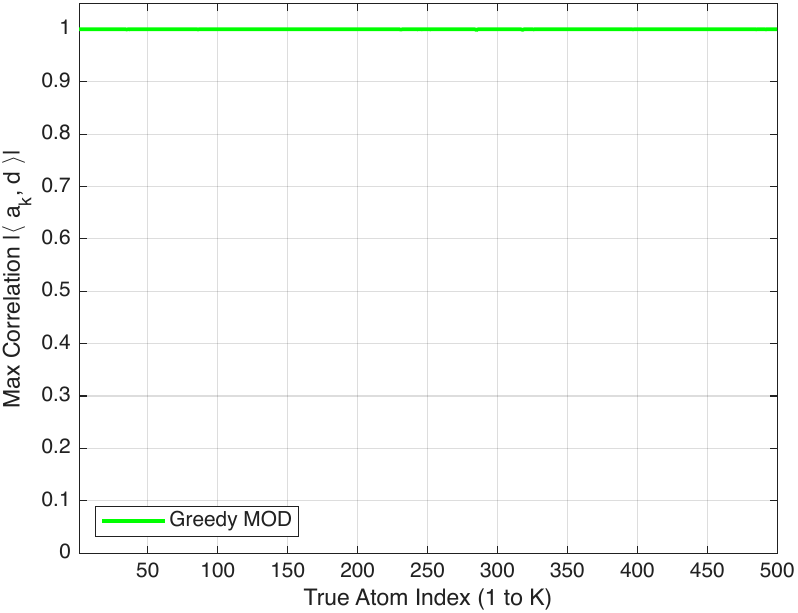}
        \caption{Greedy MOD}
        \label{fig:greedy_mod_5ini}
    \end{subfigure}
    \hfill
    \begin{subfigure}[t]{0.23\textwidth}
        \centering
        \includegraphics[width=\textwidth]{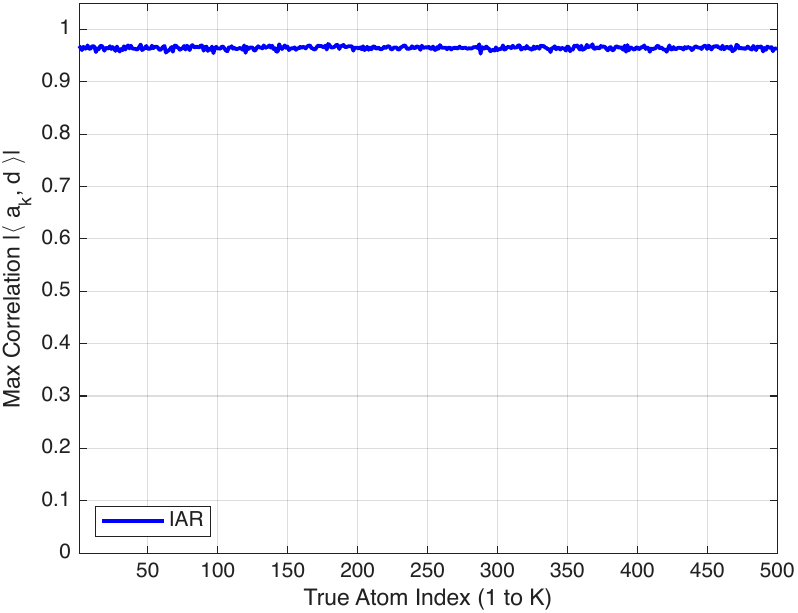}
        \caption{IAR}
        \label{fig:iar}
    \end{subfigure}
    \caption{Atom alignment across recovery methods.} 
    \label{fig:alignment_comparison}
\end{figure*}

\section{Proof of \Cref{thm:single_atom}}\label{iar_thm}

\begin{proof}
Pick a ${\bf d}^{(0)}$ and let ${\bf \tilde d}^{0}= A^T {\bf d}^{(0)}$. Then, by assumption,
\[
\tilde{d}^0_1= \langle {\bf d}^{(0)}, {\bf a}_1 \rangle > \tilde{d}^0_j = \langle {\bf d}^{(0)}, {\bf a}_j \rangle
\quad \text{for all } j>1.
\]
Without loss of generality we can assume that the components of ${\bf \tilde d}^0$ are ordered, so
\[
\tilde{d}^0_1 > \tilde{d}^0_2 > \tilde{d}^0_3 > \cdots > \tilde{d}^0_K.
\]
If ${\bf X}$ is full data, then the proportions $p_k^{(r)}$ are exactly the probabilities described in~\eqref{full_data}. Quantitative separation of the first-step probabilities is given by ~\Cref{lemma4}.

We now claim that
\[
| p_1^{(2)} - 1 | \le \frac{C}{K},
\qquad
p_k^{(2)} \le \frac{C}{K}, \quad k \ge 2,
\]
for a constant $C$ independent of $K$.
Indeed, by \Cref{lemma4}, with high probability there exist
$c\in(0,1)$, $\varepsilon>0$, and $\kappa$ such that
\[
c+(s-1)\varepsilon<1,
\]
and
\[
\tilde d_2^{(1)}\le c\,\tilde d_1^{(1)},
\qquad
\tilde d_i^{(1)}\le \varepsilon\,\tilde d_1^{(1)}
\quad\text{for all }i\ge\kappa.
\]
Let the threshold $\alpha_1$ be chosen, similar to $\alpha_0$, so that
\[
\#\Bigl\{ i : \langle {\bf d}^{(1)}, {\bf y}_i \rangle > \alpha_1 \Bigr\}
= \frac{M}{K}=\frac{1}{K}\binom{K}{s}=\Theta(K^{s-1}).
\]

We claim that the number of ${\bf y}_i$ such that
$\langle {\bf y}_i, {\bf d}^{(1)} \rangle > \alpha_1$
and index $1$ is not in the support of ${\bf x}_i$ is at most
$\kappa^2 K^{s-2}$.
Indeed, there are exactly
$\binom{K-1}{s-1}=sL$ data points ${\bf y}_i$ of the form
\begin{equation}\label{type0}
{\bf y}_i = {\bf a}_1 + \sum_{m \ge 2,\, m \in s_i} {\bf a}_m.
\end{equation}
Data points~\eqref{type0} satisfy
\begin{equation}\label{type1}
\langle {\bf d}^{(1)}, {\bf y}_i \rangle \ge \tilde{d}^{(1)}_1.
\end{equation}
It implies $\alpha_1 \ge \tilde{d}^{(1)}_1$.  So we just need to count how many ${\bf y}_i$ of the form
\begin{equation}\label{type2}
{\bf y}_i := \sum_{m \ge 2,\, m \in s_i} {\bf a}_m
\end{equation}
satisfy~\eqref{type1}.
In order for data sample of the form~\eqref{type2} to satisfy~\eqref{type1} we must have at least two indices of the support $s_i$
to be less than $\kappa$. Indeed assume it contains the next best thing, i.e., the second-best atom ${\bf a}_2$. Even in this case, the total correlation is bounded by
\[
\langle {\bf y}_i, {\bf d}^{(1)} \rangle
\le \tilde{d}_2^{(1)} + \!\!\!\! \sum_{m \ge \kappa,\, m \in s_i} \tilde{d}^{(1)}_m
\le c\,\tilde{d}_1^{(1)}+(s-1)\varepsilon\,\tilde{d}_1^{(1)}
<\tilde{d}_1^{(1)}.
\]
Now, the number of~\eqref{type2} with $m_1,m_2 \in s_i$, $2\leq m_1 \le \kappa$, $2\leq m_2 \le \kappa$ is at most
\[
\frac{s(s-1)}{2} \kappa^2 K^{s-2}.
\]
Hence, since $L=\frac{1}{K}\binom{K}{s}=\Theta(K^{s-1})$,
\[
p_1^{(2)} \ge 1- \frac{\frac{s(s-1)}{2} \kappa^2 K^{s-2}}{L}
\ge 1 - \frac{C}{K}.
\]
The number of~\eqref{type0} with index $k$ in the support is at most the number of data points whose supports contain both indices $1$ and $k$, which is $\binom{K-2}{s-2}=\mathcal{O}(K^{s-2})$. Hence,
\[
p_k^{(2)} \le \frac{\frac{s(s-1)}{2} \kappa^2 K^{s-2}+\binom{K-2}{s-2}}{L} \le \frac{C}{K}.
\]

We now claim that $p_1^{(3)} = 1$. Recall that
${\bf d}^{(2)}=\gamma A{\bf p}^{(2)}$, $\|{\bf d}^{(2)}\|_2=1$,
and
${\bf p}^{(2)} = {\bf e}_1 + {\boldsymbol\delta}$, $\|\boldsymbol\delta\|_\infty \le \frac{C}{K}$,
for some constant $C$. Since $\| {\bf p}^{(2)}\| =1+O(1/\sqrt{K})$ we also have
\[
A^T{\bf d}^{(2)} = {\bf e}_1 + {\bf v},
\qquad\text{with}\qquad
\|{\bf v}\| \le \frac{C_1}{\sqrt{K}}.
\]

Let the threshold $\alpha_2$ be chosen, similar to $\alpha_0$ and $\alpha_1$, so that
\[
\#\Bigl\{ i : \langle {\bf d}^{(2)}, {\bf y}_i \rangle > \alpha_2 \Bigr\}
= \frac{M}{K}=\frac{1}{K}\binom{K}{s}=\Theta(K^{s-1}).
\]

If $1 \in s_i$ then the correlation is close to $1$:
\[
\langle {\bf y}_i, {\bf d}^{(2)} \rangle
= \langle {\bf x}_i, {\bf e}_1 + {\bf v}  \rangle
= 1 + O\!\left(\sqrt{\frac{s}{K}}\right) \to 1, \,\, \text{as} \,\, K\to \infty.
\]
If $1 \notin s_i$ then the correlation is small:
\[
\langle {\bf y}_i, {\bf d}^{(2)} \rangle
= \langle {\bf x}_i, {\bf e}_1 + {\bf v}  \rangle
= O\!\left(\sqrt{\frac{s}{K}}\right) \to 0, \,\, \text{as} \,\, K\to \infty.
\]
The number of ${\bf y}_i$ such that $1 \in s_i$ is more than $M/K$.
Thus, when thresholding at level $\alpha_2$, only ${\bf y}_i$ with $1 \in s_i$ will be selected.
Therefore $p_1^{(3)} = 1$.

The same argument used to pass from the estimates on
\(\mathbf{p}^{(2)}\) to those on \(\mathbf{p}^{(3)}\) shows that the
corresponding estimates on \(\mathbf{p}^{(r-1)}\) imply those on
\(\mathbf{p}^{(r)}\).  The claim therefore follows by induction.
\end{proof}

\section{Proof of the Frequency Separation \Cref{lemma4}}\label{tech_lemma}

\paragraph{Step 1: Approximation of $p_k$.}
Let
\[
X_1,\dots,X_K\stackrel{\mathrm{i.i.d.}}{\sim}\mathcal N(0,1),
\qquad
X_{(1)}\ge\cdots\ge X_{(K)}
\]
be the corresponding order statistics. After relabeling the dictionary
atoms, we may write
\[
\widetilde d_k=X_{(k)}.
\]
Recall that
\[
q_K=\Phi^{-1}\!\left(1-\frac1K\right),
\qquad
\bar\alpha_0=\sqrt{s}\,q_K,
\]
and define
\[
g_K(x)
:=
\bar\Phi\!\left(
\frac{\bar\alpha_0-x}{\sqrt{s-1}}
\right).
\]

We claim that, for every fixed $\kappa$,
\begin{equation}\label{eq:conditional-pk}
p_k
=
g_K(X_{(k)})\bigl(1+o_{\mathbb P}(1)\bigr),
\qquad
k=1,\dots,\kappa+1,
\end{equation}
uniformly over these finitely many indices.

We first compare the random threshold $\alpha_0$ with $\bar\alpha_0$.
For $t\in\mathbb R$, set
\[
F_{s,K}(t)
:=
\frac{1}{\binom Ks}
\sum_{\substack{S\subset\{1,\dots,K\}\\ |S|=s}}
\mathbf 1
\left\{
\sum_{j\in S}X_j\ge t
\right\}.
\]
Then
\[
\mathbb E F_{s,K}(t)
=
\bar\Phi\!\left(\frac{t}{\sqrt{s}}\right).
\]
In particular, for fixed $u\in\mathbb R$,
\[
\mathbb E
F_{s,K}\left(\bar\alpha_0+\frac{u}{q_K}\right)
=
\frac1K
\exp\left\{-\frac{u}{\sqrt{s}}+o(1)\right\}.
\]

Moreover,
\begin{equation}\label{eq:Ustat-concentration}
\frac{
\operatorname{Var}
F_{s,K}\left(\bar\alpha_0+\frac{u}{q_K}\right)
}{
\left(
\mathbb E
F_{s,K}\left(\bar\alpha_0+\frac{u}{q_K}\right)
\right)^2
}
\longrightarrow0,
\end{equation}
uniformly for bounded $u$. Indeed, two $s$-element subsets with
intersection of cardinality $r$ occur with relative frequency
$O(K^{-r})$, while the corresponding normalized Gaussian sums have
correlation $r/s$. Standard bivariate Gaussian tail estimates give
\[
\mathbb P
\left\{
Z_1\ge q_K+O(q_K^{-1}),
\;
Z_2\ge q_K+O(q_K^{-1})
\right\}
=
K^{-2s/(s+r)+o(1)}.
\]
Thus the contribution of pairs with intersection of cardinality $r$ to
the relative variance in~\eqref{eq:Ustat-concentration} is
\[
K^{\,2-r-\frac{2s}{s+r}+o(1)}
=o(1),
\qquad
r=1,\dots,s.
\]

Since the actual threshold $\alpha_0$ satisfies
\[
F_{s,K}(\alpha_0)=\frac1K,
\]
Chebyshev's inequality, applied at
$\bar\alpha_0\pm u/q_K$, gives
\begin{equation}\label{eq:alpha-concentration}
\alpha_0
=
\bar\alpha_0+o_{\mathbb P}(q_K^{-1}).
\end{equation}

We now consider $p_k$. Set
\[
F_{s-1,K}(t)
:=
\frac{1}{\binom K{s-1}}
\sum_{\substack{T\subset\{1,\dots,K\}\\ |T|=s-1}}
\mathbf 1
\left\{
\sum_{j\in T}X_j\ge t
\right\}.
\]
Conditional on the realization of the $X_j$'s, the probability $p_k$ is
obtained by averaging over all $(s-1)$-element subsets not containing the
index corresponding to $X_{(k)}$. Hence
\[
p_k
=
F_{s-1,K}\bigl(\alpha_0-X_{(k)}\bigr)
+
O\left(\frac1K\right).
\]

For every fixed $k$,
\[
X_{(k)}
=
q_K+O_{\mathbb P}(q_K^{-1}),
\]
and therefore
\[
\alpha_0-X_{(k)}
=
(\sqrt{s}-1)q_K+O_{\mathbb P}(q_K^{-1}).
\]
The same second-moment argument as above, now for $(s-1)$-element subsets,
gives
\[
F_{s-1,K}(t)
=
\bar\Phi\!\left(\frac{t}{\sqrt{s-1}}\right)
\bigl(1+o_{\mathbb P}(1)\bigr)
\]
uniformly for
\[
t=(\sqrt{s}-1)q_K+O(q_K^{-1}).
\]
Using~\eqref{eq:alpha-concentration}, we obtain
\[
p_k
=
g_K(X_{(k)})
\bigl(1+o_{\mathbb P}(1)\bigr)
+
O\left(\frac1K\right).
\]
Finally,
\[
g_K(X_{(k)})
=
K^{-c(s)+o_{\mathbb P}(1)},
\qquad
c(s)
=
\frac{(\sqrt{s}-1)^2}{s-1}
=
\frac{\sqrt{s}-1}{\sqrt{s}+1}
<1.
\]
Thus the $O(1/K)$ term is negligible, and
\eqref{eq:conditional-pk} follows.

\paragraph{Step 2: Separation of the leading probabilities.}
Set
\[
p:=\frac{1}{\sqrt{s}+1}.
\]
By Appendix~\ref{app:successive-pk}, \Cref{thm:pj_over_g},
for every fixed $\kappa$,
\[
\left(
\frac{g_K(X_{(2)})}{g_K(X_{(1)})},
\frac{g_K(X_{(\kappa)})}{g_K(X_{(1)})}
\right)
\Longrightarrow
\left(
\left(\frac{S_1}{S_2}\right)^p,
\left(\frac{S_1}{S_\kappa}\right)^p
\right),
\]
where
\[
S_j=E_1+\cdots+E_j
\]
and $E_1,E_2,\dots$ are independent $\operatorname{Exp}(1)$ random
variables. Together with~\eqref{eq:conditional-pk}, this gives
\begin{equation}\label{eq:pk-ratio-limit}
\left(
\frac{p_2}{p_1},
\frac{p_\kappa}{p_1}
\right)
\Longrightarrow
\left(
\left(\frac{S_1}{S_2}\right)^p,
\left(\frac{S_1}{S_\kappa}\right)^p
\right).
\end{equation}

It suffices to consider $0<\eta<1$. Since
\[
\frac{S_1}{S_2}
=
\frac{E_1}{E_1+E_2}
\]
is uniformly distributed on $(0,1)$, we may choose $c\in(0,1)$ so close
to $1$ that
\[
\mathbb P
\left\{
\left(\frac{S_1}{S_2}\right)^p<c
\right\}
>
1-\frac{\eta}{3}.
\]
Choose $\varepsilon>0$ so that
\[
c+(s-1)\varepsilon<1.
\]
Furthermore,
\[
\frac{S_1}{S_\kappa}
\sim\operatorname{Beta}(1,\kappa-1),
\]
and hence
\[
\mathbb P
\left\{
\left(\frac{S_1}{S_\kappa}\right)^p
\ge\varepsilon
\right\}
=
\left(1-\varepsilon^{1/p}\right)^{\kappa-1}.
\]
We may therefore choose $\kappa$ sufficiently large that this probability
is less than $\eta/3$.

It follows from~\eqref{eq:pk-ratio-limit} that, for all sufficiently large
$K$, with probability at least $1-\eta$,
\[
p_2\le c\,p_1,
\qquad
p_\kappa\le\varepsilon p_1.
\]
Since $p_k^{(1)}=s\,p_k$, the same inequalities hold for
$p_k^{(1)}$. This proves \Cref{lemma4}.
\qed

\section{Conclusions}\label{sec:conclusions}

The main contribution of this work is a monotonicity principle for dictionary
learning. We introduce atom-selection frequencies that describe how often
each generating atom appears among the samples selected by the current
iterate. Although these quantities are not observed by the algorithm, they
provide an effective description of its dynamics. In the model considered
here, their ordering is preserved under iteration, while a strict advantage
of one atom is amplified from one step to the next.

This principle turns the recovery problem into the study of a simple
one-dimensional ordering mechanism. Under the orthogonal full-data model,
it implies that a randomly initialized iterate isolates a generating atom
after three refinement steps. Repeating the procedure from independent
random initializations then gives recovery of the entire dictionary at the coupon-collector scale

$$
R\sim K\log K.
$$

The monotonicity principle is not merely a technical feature of the idealized
model. We view it as the main structural mechanism uncovered in this work.
It suggests a different way to analyze dictionary-learning dynamics. Instead
of tracking the full nonconvex iteration, one follows the ordering and
amplification of atom-selection frequencies.

Our numerical experiments indicate that this mechanism persists far beyond
the setting in which it is proved, including incomplete data and
non-orthogonal dictionaries. It may also provide a theoretical mechanism for
the successful recovery from random initialization observed in our earlier
neural-network experiments~\cite{Weak}, although establishing this connection
rigorously remains an open problem.

We therefore expect monotonicity of selection frequencies to be a robust
principle in dictionary learning, rather than an artifact of the present
model. A central question is to determine how broadly this principle can be
established rigorously.

\appendix

\section{Appendix. Asymptotic separation for successive probabilities}\label{app:successive-pk}
The next Theorem provides the exact asymptotic scaling for the selection probabilities in the IAR algorithm, proving precisely how the target atom detaches itself from the background 
noise during the initial thresholding step.

\begin{theorem}\label{thm:pj_over_g}
Fix an integer $s\ge 2$ and an integer $j\ge 1$. Let
$X_1,\dots,X_K$ be i.i.d.\ $N(0,1)$ with decreasing order statistics
\[
X_{(1)}\ge\cdots\ge X_{(K)}.
\]
Let $Y\sim N(0,s-1)$ be independent. Let $\Phi$ denote the CDF of
$N(0,1)$ and $\bar\Phi:=1-\Phi$. Define
\[
q_K:=\Phi^{-1}\!\left(1-\frac1K\right),
\qquad
\alpha_0:=\sqrt{s}\,q_K,
\]
and
\[
g(x):=\bar\Phi\!\left(\frac{\alpha_0-x}{\sqrt{s-1}}\right).
\]
Set
\[
p=\frac{1}{\sqrt{s}+1}.
\]
Then, for every fixed $j$,
\[
\left(
\frac{g(X_{(1)})}{g(q_K)},
\dots,
\frac{g(X_{(j)})}{g(q_K)}
\right)
\Longrightarrow
\left(
S_1^{-p},
\dots,
S_j^{-p}
\right),
\]
where
\[
S_k=E_1+\cdots+E_k
\]
and $E_1,E_2,\dots$ are i.i.d.\ $\operatorname{Exp}(1)$ random variables.

Moreover, if
\[
\widetilde p_j
:=
\PP\bigl(Y+X_{(j)}\ge\alpha_0\bigr),
\]
then
\[
\frac{\widetilde p_j}{g(q_K)}
\longrightarrow
\frac{\Gamma(j-p)}{\Gamma(j)}.
\]
\end{theorem}

\begin{proof} 

For the entire proof the constants $j$, $s$, and therefore $p$ are fixed.
Conditioning on $X_{(j)}$ gives
\[
\widetilde p_j
=
\mathbb{P}\left(Y+X_{(j)}\ge\alpha_0\right)
=
\EE[g(X_{(j)})].
\]
Define, for $u\in\mathbb{R}$,
\[
r_K(u):=
\frac{g\!\left(q_K+\frac{u}{q_K}\right)}{g(q_K)},
\]
and, for $k=1,\dots,j$,
\[
U_{k,K}:=q_K\bigl(X_{(k)}-q_K\bigr).
\]
Then
\[
\frac{\widetilde p_j}{g(q_K)}
=
\EE\big[r_K(U_{j,K})\big].
\]

We will use the standard joint extreme-value limit for Gaussian order
statistics; see, for example,
\cite[Chs.~3--5]{Resnick1987} or
\cite[Chs.~1--2]{Leadbetter1983}:
\[
\bigl(U_{1,K},\dots,U_{j,K}\bigr)
\Rightarrow
\bigl(U_1,\dots,U_j\bigr),
\qquad
U_k=-\log S_k,
\qquad
S_k=E_1+\cdots+E_k,
\]
where $E_1,E_2,\dots$ are i.i.d.\ $\mathrm{Exp}(1)$ random variables.

\medskip
\noindent\textbf{Step 1 (compact convergence and the bounding envelope \(r_K\le G_K\)).}
Fix $\Lambda>0$ to split the analysis within the central region
$|u|\le\Lambda$ and outside it. Since $\alpha_0=\sqrt{s}\,q_K$, we have
\[
\frac{\alpha_0-q_K}{\sqrt{s-1}}
=
\frac{(\sqrt{s}-1)q_K}{\sqrt{s-1}}
\]
and
\[
\frac{\alpha_0-q_K}{\sqrt{s-1}}
-
\frac{u}{\sqrt{s-1}\,q_K}
=
\frac{(\sqrt{s}-1)q_K-\frac{u}{q_K}}{\sqrt{s-1}}.
\]
Hence for $|u|\le\Lambda$,
\[
\log r_K(u)
=
\log\bar\Phi\!\left(
\frac{(\sqrt{s}-1)q_K-\frac{u}{q_K}}{\sqrt{s-1}}
\right)
-
\log\bar\Phi\!\left(
\frac{(\sqrt{s}-1)q_K}{\sqrt{s-1}}
\right).
\]
Using the logarithmic Mills expansion
\[
\log\bar\Phi(z)
=
-\frac{z^2}{2}-\log z-\frac12\log(2\pi)+O(z^{-2})
\qquad (z\to\infty),
\]
applied to the two arguments
\[
\frac{(\sqrt{s}-1)q_K}{\sqrt{s-1}}
\quad\text{and}\quad
\frac{(\sqrt{s}-1)q_K-\frac{u}{q_K}}{\sqrt{s-1}},
\]
we obtain
\begin{align}
\log r_K(u)
&=
-\frac{1}{2}\left[
\left(
\frac{(\sqrt{s}-1)q_K-\frac{u}{q_K}}{\sqrt{s-1}}
\right)^2
-
\left(
\frac{(\sqrt{s}-1)q_K}{\sqrt{s-1}}
\right)^2
\right]
\notag\\
&\quad
-\log\!\left(
1-\frac{u}{(\sqrt{s}-1)q_K^2}
\right)
+O(q_K^{-2})
=
pu+o(1),
\qquad |u|\le\Lambda .
\label{eq:local_log_rK}
\end{align}
Exponentiating gives
\begin{equation}\label{eq:local_rK}
\sup_{|u|\le\Lambda}
\bigl|r_K(u)-e^{pu}\bigr|
\to0.
\end{equation}

Since
\[
\bigl(U_{1,K},\dots,U_{j,K}\bigr)
\Rightarrow
\bigl(U_1,\dots,U_j\bigr)
\]
and this sequence is tight, the uniform convergence
in~\eqref{eq:local_rK} gives
\[
\bigl(
r_K(U_{1,K}),\dots,r_K(U_{j,K})
\bigr)
\Rightarrow
\bigl(
e^{pU_1},\dots,e^{pU_j}
\bigr)
=
\bigl(
S_1^{-p},\dots,S_j^{-p}
\bigr).
\]
Since
\[
r_K(U_{k,K})
=
\frac{g(X_{(k)})}{g(q_K)},
\]
this proves the first assertion of the theorem.

It remains to justify the passage to expectations in the $j$-th component.

Next define the piecewise global majorant $G_K(u)$ over three distinct regions of the real line
\[
G_K(u):=
\begin{cases}
1, & u\le 0,\\[1mm]
C e^{pu}, & 0<u\le \dfrac{\sqrt{s}-1}{2}\,q_K^2,\\[2mm]
g(q_K)^{-1}, & u> \dfrac{\sqrt{s}-1}{2}\,q_K^2.
\end{cases}
\]
We claim that for all large \(K\), 
\begin{equation}\label{eq:rK_le_GK}
r_K(u)\le G_K(u)\qquad (u\in\mathbb{R}).
\end{equation}

Indeed, since \(g\) is increasing,
\[
r_K(u)\le 1\qquad (u\le 0).
\]
Also,  the same comparison of the two Gaussian tails, using only upper and lower Mills bounds, gives
\[
r_K(u)\le C e^{pu}
\qquad\text{for }0<u\le \frac{\sqrt{s}-1}{2}\,q_K^2.
\]
Finally, if \(u>\frac{\sqrt{s}-1}{2}\,q_K^2\), then \(r_K(u)\le g(q_K)^{-1}\). Write
\[
c:=\frac{\sqrt{s}-1}{\sqrt{s-1}}\in(0,1),
\qquad
g(q_K)=\bar\Phi(cq_K).
\]
Since \(\bar\Phi(q_K)=1/K\),
\[
\frac{1}{K\,g(q_K)}
=
\frac{\bar\Phi(q_K)}{\bar\Phi(cq_K)}\to 0
\qquad (K\to\infty)
\]
by Mills' ratio, because \(c<1\). Hence \(g(q_K)^{-1}\le K\) for all large \(K\), and \eqref{eq:rK_le_GK} follows.

\medskip
\noindent\textbf{Step 2 (upper tail).} 
For \(u\ge 0\),
\[
\{U_{j,K}>u\}\subseteq \{U_{1,K}>u\}
=\bigcup_{i=1}^K \left\{X_i>q_K+\frac{u}{q_K}\right\}.
\]
Therefore
\[
\PP(U_{j,K}>u)\le K\,\bar\Phi\!\left(q_K+\frac{u}{q_K}\right).
\]
Since \(\bar\Phi(q_K)=1/K\),
\[
K\,\bar\Phi\!\left(q_K+\frac{u}{q_K}\right)
=
\frac{\bar\Phi(q_K+u/q_K)}{\bar\Phi(q_K)}.
\]
Moreover,
\[
\log\frac{\bar\Phi(q_K+u/q_K)}{\bar\Phi(q_K)}
=
-\int_{q_K}^{q_K+u/q_K}\frac{\phi(t)}{\bar\Phi(t)}\,dt
\le
-\int_{q_K}^{q_K+u/q_K} t\,dt
\le -u,
\]
where  standard normal density $\phi(t)$ is the derivative of the Gaussian survival function  $\bar\Phi(t)$. Here, we used \(\phi(t)/\bar\Phi(t)\ge t\) for \(t>0\). Hence
\begin{equation}\label{eq:upper_tail_basic}
\PP(U_{j,K}>u)\le \frac{\bar\Phi(q_K+u/q_K)}{\bar\Phi(q_K)} \le e^{-u}\qquad (u\ge 0).
\end{equation}

Fix \(\Lambda>0\). By \eqref{eq:rK_le_GK},
\begin{align*}
&\EE\!\left[r_K(U_{j,K})\mathbf 1_{\{U_{j,K}>\Lambda\}}\right]
\le
\EE\!\left[G_K(U_{j,K})\mathbf 1_{\{U_{j,K}>\Lambda\}}\right] \\
&\le
C\,\EE\!\left[e^{pU_{j,K}}\mathbf 1_{\{\Lambda<U_{j,K}\le \frac{\sqrt{s}-1}{2}q_K^2\}}\right]
+
g(q_K)^{-1}\PP\!\left(U_{j,K}>
\frac{\sqrt{s}-1}{2}q_K^2\right).
\end{align*}
For the first term, using \eqref{eq:upper_tail_basic},
\begin{align*}
\EE\!\left[e^{pU_{j,K}}\mathbf 1_{\{U_{j,K}>\Lambda\}}\right]
&=
e^{p\Lambda}\PP(U_{j,K}>\Lambda)
+
p\int_\Lambda^\infty e^{pu}\PP(U_{j,K}>u)\,du \\
&\le
e^{p\Lambda}e^{-\Lambda}
+
p\int_\Lambda^\infty e^{pu}e^{-u}\,du = \frac{1}{1-p}e^{-(1-p)\Lambda}.
\end{align*}
Since \(p<1\), the right-hand side tends to \(0\) as \(\Lambda\to\infty\), uniformly in \(K\).

For the second term, since
\[
g(q_K)=\bar\Phi(cq_K),
\qquad
c=\frac{\sqrt{s}-1}{\sqrt{s-1}},
\]
Mills' ratio gives
\[
g(q_K)^{-1}=K^{c^2+o(1)}.
\]
Therefore,
\[
g(q_K)^{-1}
\PP\!\left(U_{j,K}>
\frac{\sqrt{s}-1}{2}q_K^2\right)
\le
K^{c^2-(\sqrt{s}-1)+o(1)}
\to0,
\]
since
\[
c^2=\frac{\sqrt{s}-1}{\sqrt{s}+1}<\sqrt{s}-1.
\]
 Thus
\begin{equation}\label{eq:upper_tail_vanish}
\lim_{\Lambda\to\infty}\ \limsup_{K\to\infty}
\EE\!\left[r_K(U_{j,K})\mathbf 1_{\{U_{j,K}>\Lambda\}}\right]=0.
\end{equation}

\medskip
\noindent\textbf{Step 3 (lower tail).}
Because \(g\) is increasing,
\[
r_K(u)\le 1\qquad (u\le 0).
\]
Hence
\[
\EE\!\left[r_K(U_{j,K})\mathbf 1_{\{U_{j,K}<-\Lambda\}}\right]
\le
\PP(U_{j,K}\le -\Lambda).
\]
Since \(U_{j,K}\Rightarrow U_j\), Portmanteau yields
\[
\limsup_{K\to\infty}\PP(U_{j,K}\le -\Lambda)
\le
\PP(U_j\le -\Lambda).
\]
Now
\[
U_j=-\log S_j,
\qquad S_j=E_1+\cdots+E_j,
\]
so
\[
\PP(U_j\le -\Lambda)=\PP(S_j\ge e^\Lambda)\to 0
\qquad  \text{ as } \Lambda\to\infty.
\]
Therefore
\begin{equation}\label{eq:lower_tail_vanish}
\lim_{\Lambda\to\infty}\ \limsup_{K\to\infty}
\EE\!\left[r_K(U_{j,K})\mathbf 1_{\{U_{j,K}<-\Lambda\}}\right]=0.
\end{equation}

\medskip
\noindent\textbf{Step 4 (evaluate the limit constant).}
Fix \(\Lambda>0\). By \eqref{eq:local_rK},
\[
\left|
\EE\!\left[r_K(U_{j,K})\mathbf 1_{\{|U_{j,K}|\le \Lambda\}}\right]
-
\EE\!\left[e^{pU_{j,K}}\mathbf 1_{\{|U_{j,K}|\le \Lambda\}}\right]
\right|
\le
\sup_{|u|\le \Lambda}\bigl|r_K(u)-e^{pu}\bigr|
\to 0.
\]
Since \(U_{j,K}\Rightarrow U_j\) and \(U_j\) has a continuous law, the bounded function
\(u\mapsto e^{pu}\mathbf 1_{\{|u|\le \Lambda\}}\) is \(U_j\)-a.s.\ continuous, and therefore
\[
\EE\!\left[e^{pU_{j,K}}\mathbf 1_{\{|U_{j,K}|\le \Lambda\}}\right]
\to
\EE\!\left[e^{pU_j}\mathbf 1_{\{|U_j|\le \Lambda\}}\right].
\]
Together with \eqref{eq:upper_tail_vanish} and \eqref{eq:lower_tail_vanish}, this yields
\[
\EE[r_K(U_{j,K})]\to \EE[e^{pU_j}]\quad \text{as} \quad K \to\infty.
\]

Since \(U_j=-\log S_j\), we have \(e^{pU_j}=S_j^{-p}\). With \(S_j\sim\mathrm{Gamma}(j,1)\),
\[
\EE[S_j^{-p}]
=
\frac{1}{\Gamma(j)}\int_0^\infty x^{j-p-1}e^{-x}\,dx
=
\frac{\Gamma(j-p)}{\Gamma(j)}
\]
Therefore
\[
\frac{\widetilde p_j}{g(q_K)}=\EE[r_K(U_{j,K})]\longrightarrow \frac{\Gamma(j-p)}{\Gamma(j)} \quad \text{as} \quad K \to\infty.
\]
\end{proof}

\section*{Acknowledgments}
The work of A. Christie and G. Papanicolaou was partially supported by AFOSR FA9550-23-1-0352.
The work of M. Moscoso was supported by  the Spanish AEI grant PID2020-115088RB-I00.
The work of A. Novikov was partially supported by NSF 2407046, AFOSR FA9550-23-1-0352, and FA9550-23-1-0523. 
The work of  C. Tsogka was partially supported by AFOSR FA9550-23-1-0352 and FA9550-24-1-0191. 

\bibliographystyle{siamplain}
\bibliography{references}

\end{document}